\documentclass{article}

\PassOptionsToPackage{numbers, sort&compress}{natbib}

\usepackage[preprint]{neurips_2026}

\usepackage[utf8]{inputenc} % allow utf-8 input
\usepackage[T1]{fontenc}    % use 8-bit T1 fonts
\usepackage{hyperref}       % hyperlinks
\usepackage{url}            % simple URL typesetting
\usepackage{booktabs}       % professional-quality tables
\usepackage{amsfonts}       % blackboard math symbols
\usepackage{nicefrac}       % compact symbols for 1/2, etc.
\usepackage{microtype}      % microtypography
\usepackage{xcolor}         % colors
\usepackage{amsmath}
\usepackage{caption}
\usepackage{adjustbox}
\usepackage{graphicx}
\usepackage{multirow}
\usepackage{makecell}
\usepackage[table]{xcolor}
\usepackage{subcaption}
\usepackage{tikz}
\usepackage{arydshln}
\usepackage{dashrule}
\usepackage{pifont}
\usepackage{wrapfig}
\newtheorem{theorem}{Theorem}[section]

\newtheorem{proposition}[theorem]{Proposition}

\newtheorem{assumption}[theorem]{Assumption}
\definecolor{lightblue}{RGB}{230,242,255}
\definecolor{teaserblue}{RGB}{25,90,180}

\title{PermuQuant: Lowering Per-Group Quantization Error by Reordering Channels for Diffusion Models}

\author{
  Yongsen Cheng$^{1}$\thanks{Equal contribution.},\enspace 
  Kai Liu$^{1}$\footnotemark[1],\enspace 
  Kaiwen Tao$^{1}$,\enspace 
  Junxian Li$^{1}$,\enspace 
  Zhixin Wang$^{2}$,\enspace 
  Zhikai Chen$^{2}$,\enspace \\
  \textbf{Renjing Pei$^{2}$,}\enspace 
  \textbf{Yulun Zhang}$^{1}$\thanks{Corresponding author: Yulun Zhang, yulun100@gmail.com} \\
  \textsuperscript{1}Shanghai Jiao Tong University,\enspace
  \textsuperscript{2}Huawei Noah's Ark Lab
}

\begin{document}

\maketitle

\begin{center}
\vspace{-3mm}

\setlength{\tabcolsep}{2pt}
\renewcommand{\arraystretch}{0.95}

\renewcommand\cellalign{c}
\renewcommand\cellgape{}

\newcommand{\teaseritem}[2]{%
    \begin{tabular}{@{}c@{}}
        {\scriptsize\renewcommand{\arraystretch}{0.75}\makecell{#1}}\\[0.0mm]
        \includegraphics[width=0.24\textwidth]{#2}
    \end{tabular}
}

\newcommand{\teaserprompt}[1]{%
    {\scriptsize
    \begin{minipage}{0.8\textwidth}
        \centering
        \setlength{\baselineskip}{1mm}
        \textit{Prompt:} #1
    \end{minipage}}
}

\begin{tabular}
{@{}c@{\hspace{0mm}}c@{\hspace{0mm}}c@{\hspace{0mm}}c@{}}
\vspace{-2mm}
    \teaseritem{FLUX.1-dev BF16\\(50 Steps)\\DiT Mem.: 22.2 GiB\\E2E Lat.: 81.4 s}{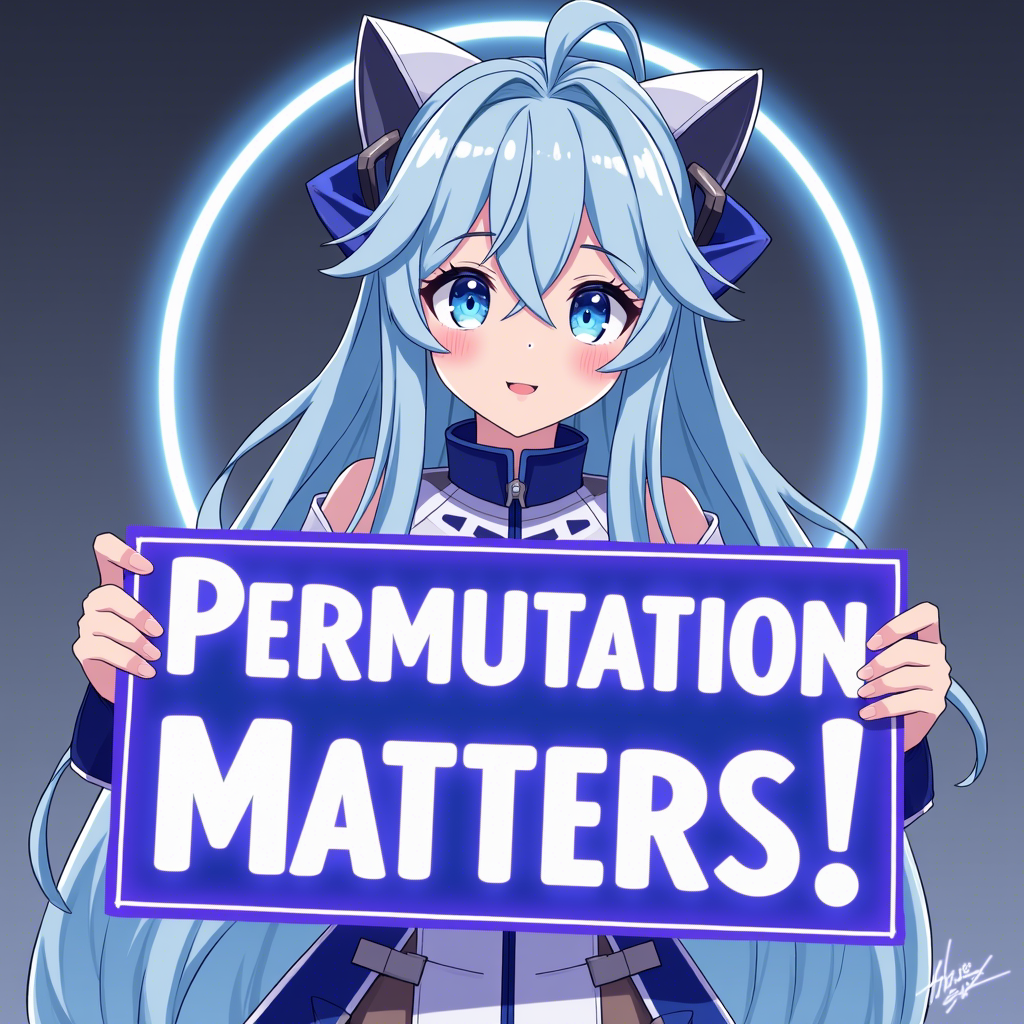} &
    \teaseritem{NF4 (W4A16)\\LPIPS: 0.263\\DiT Mem.: 5.6 GiB (4.0$\times$ Less)\\E2E Lat.: 23.3 s (3.8$\times$ Faster)}{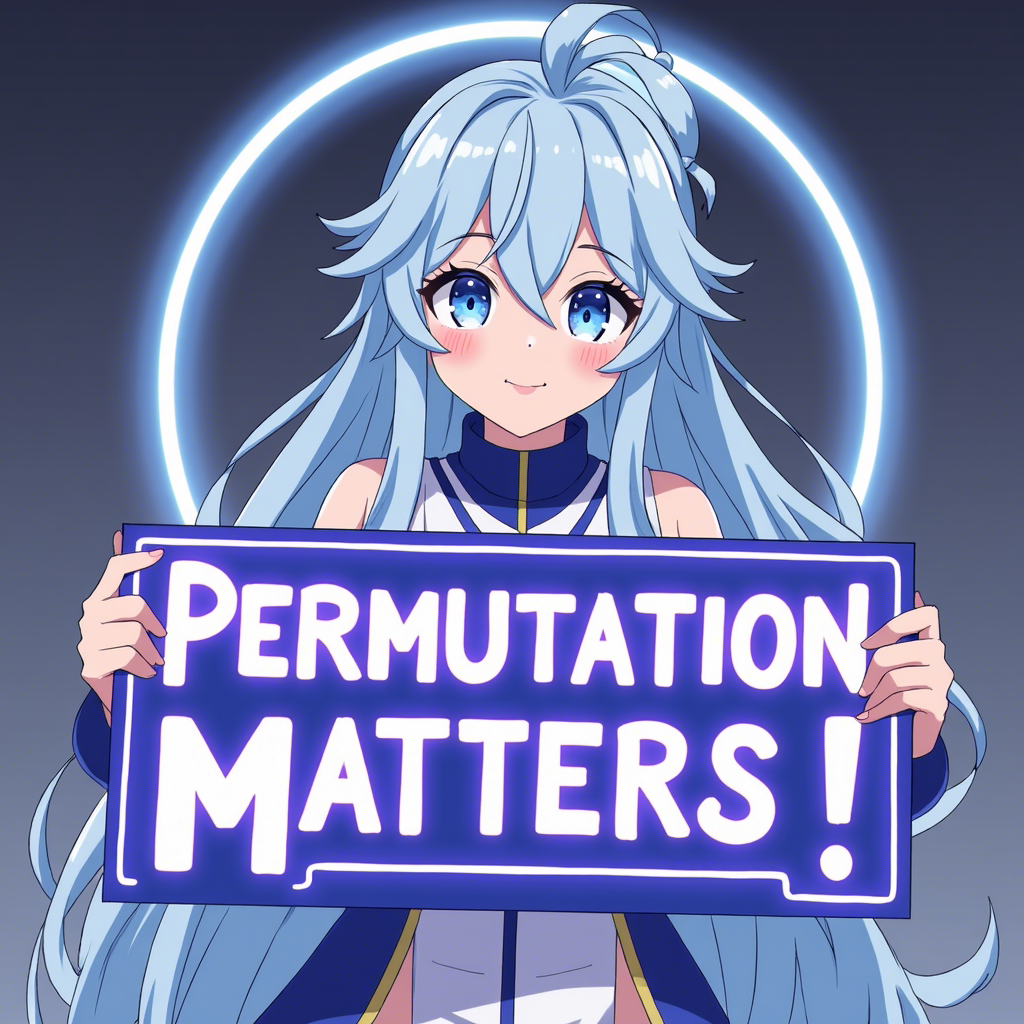} &
    \teaseritem{NVFP4 (W4A4)\\LPIPS: 0.596\\DiT Mem.: 6.3 GiB (3.5$\times$ Less)\\E2E Lat.: 13.1 s (6.2$\times$ Faster)}{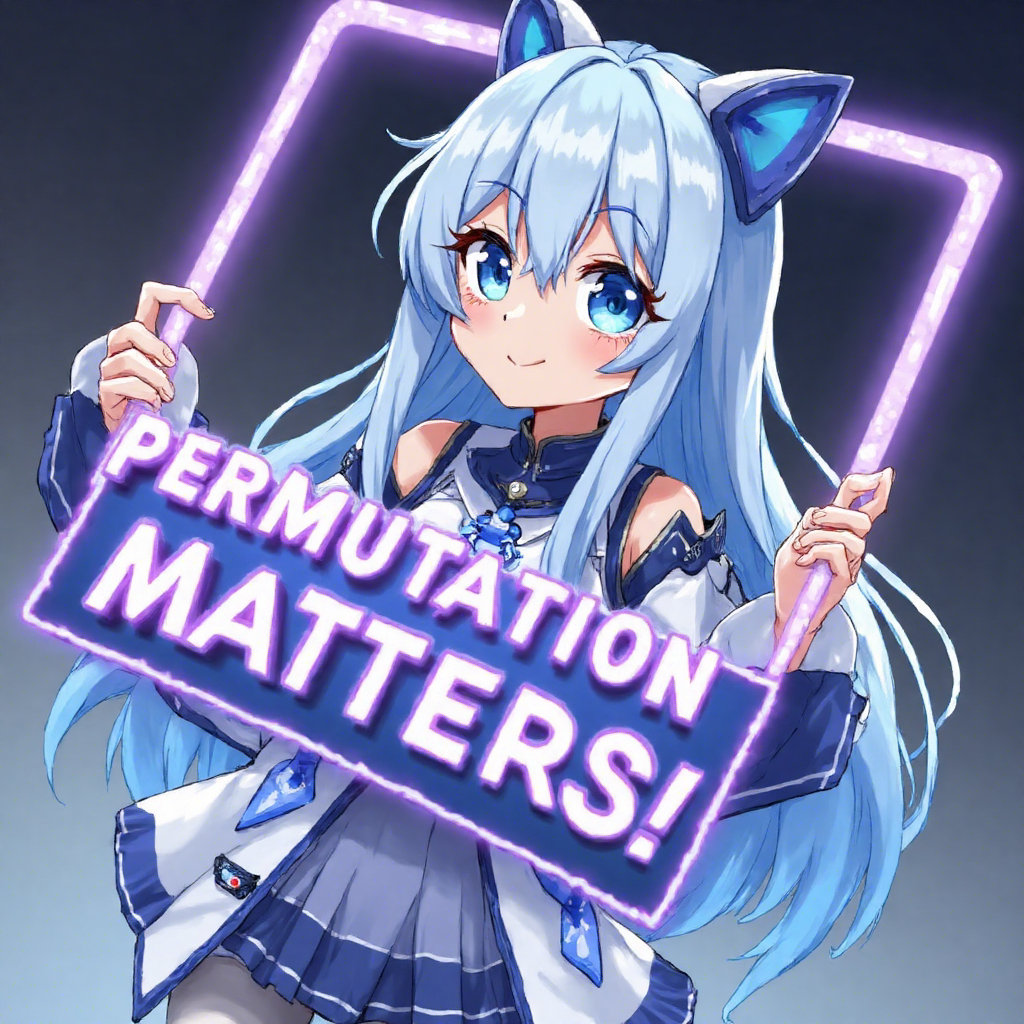} &
    \teaseritem{PermuQuant NVFP4 (W4A4)\\LPIPS: 0.186\\DiT Mem.: \textcolor{teaserblue}{6.3 GiB (3.5$\times$ Less)}\\E2E Lat.: \textcolor{teaserblue}{13.0 s (6.3$\times$ Faster)}}{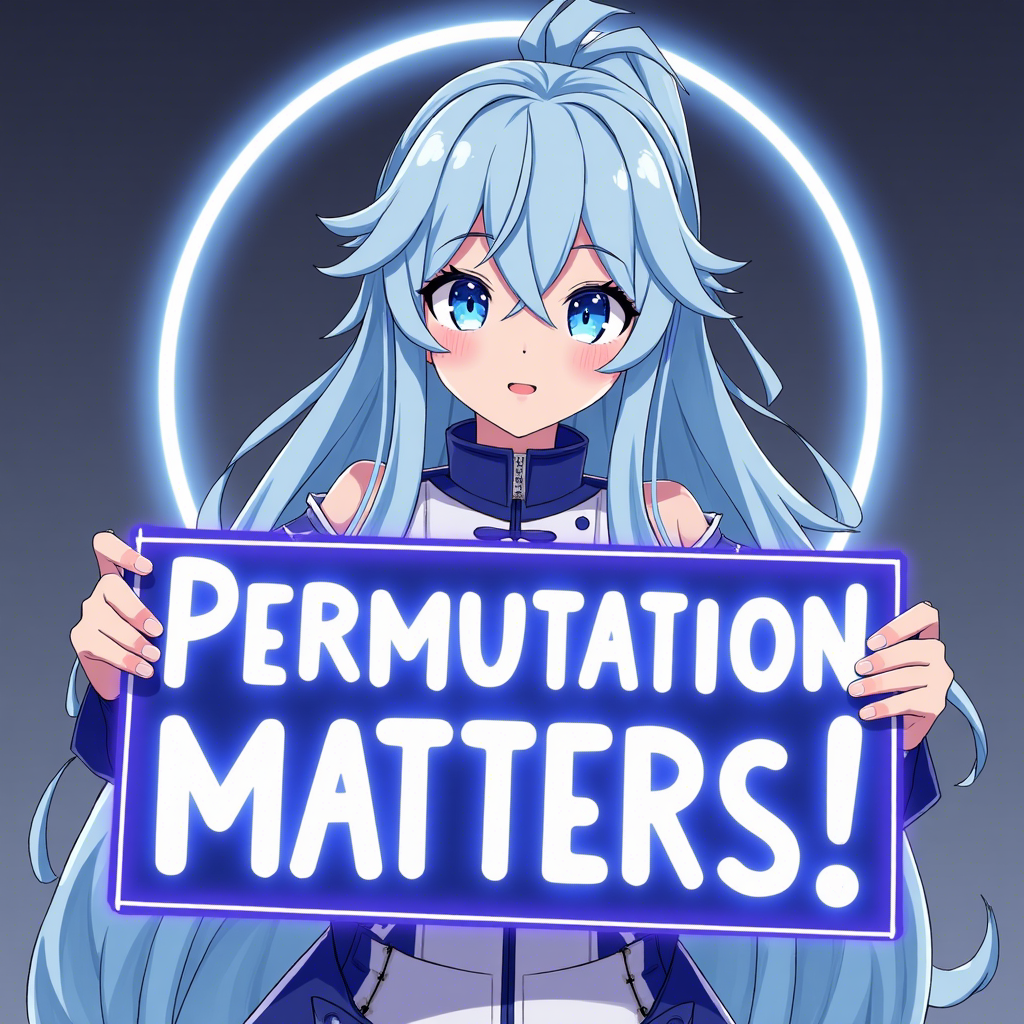}
\end{tabular}

\vspace{-1mm}

\teaserprompt{A beautiful anime-style girl, upper body, with long light-blue hair and blue eyes, wearing a white and blue outfit, holding a huge neon sign that says ``Permutation Matters!''}

\vspace{-3mm}
\noindent\hdashrule{1.00\textwidth}{0.3pt}{1.5pt}
\vspace{-5mm}

\begin{tabular}{@{}c@{\hspace{0mm}}c@{\hspace{0mm}}c@{\hspace{0mm}}c@{}}
    \teaseritem{FLUX.1-dev BF16\\(50 Steps)}{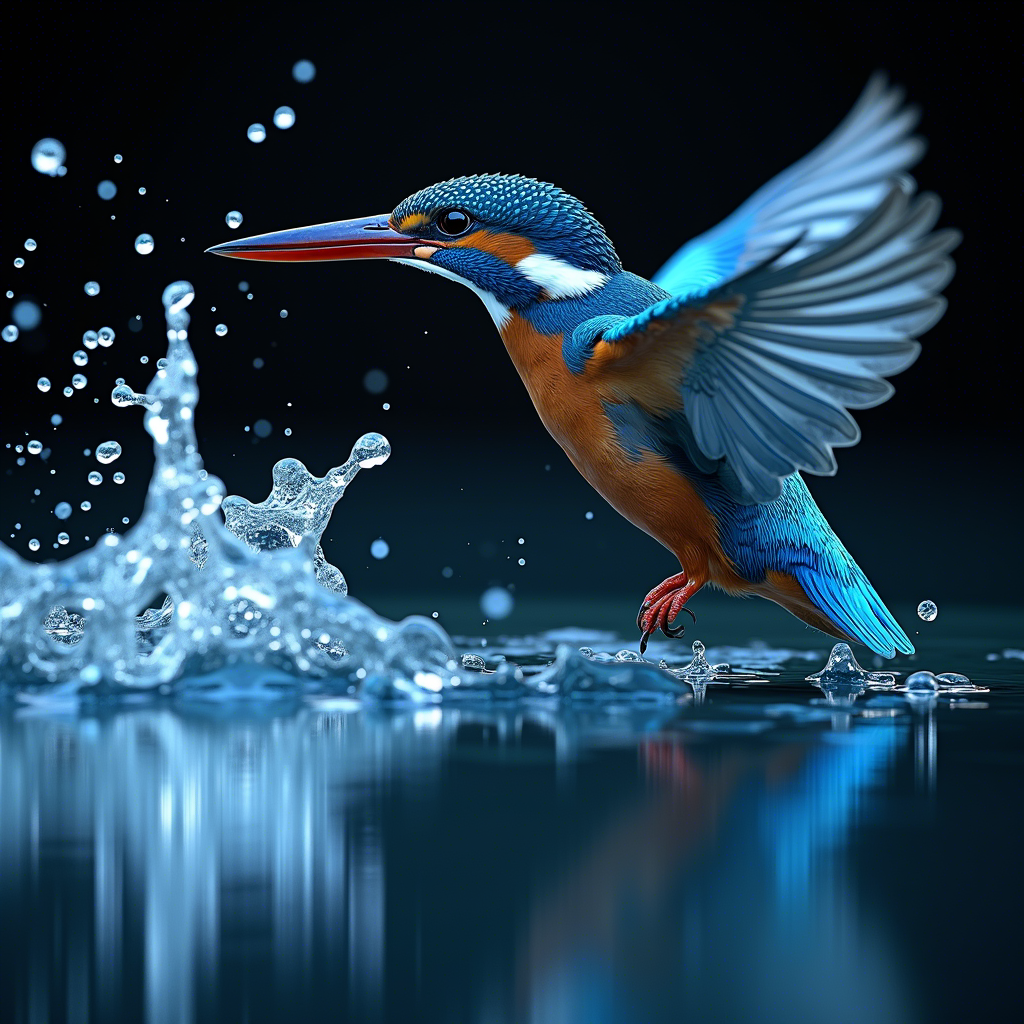} &
    \teaseritem{OmniQuant (W3A3)\\LPIPS: 0.673}{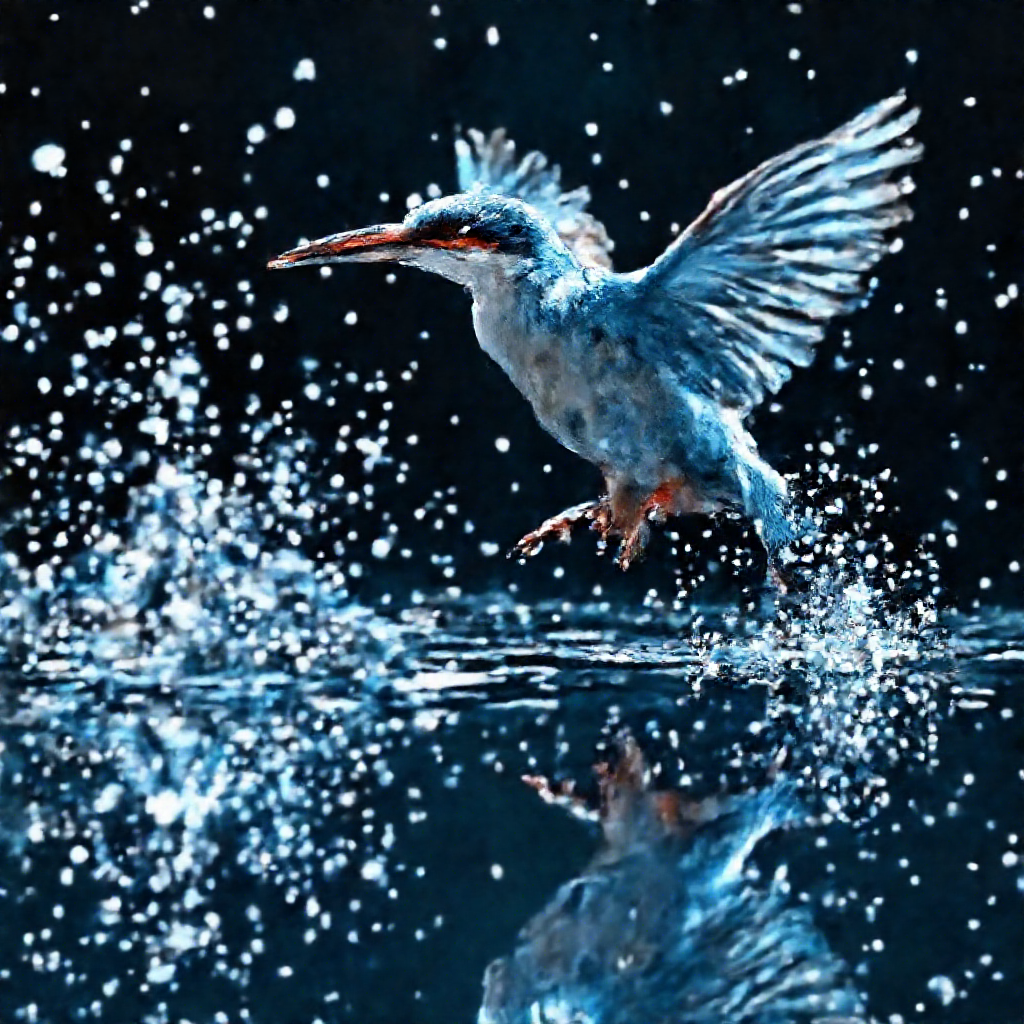} &
    \teaseritem{ConvRot (W3A3)\\LPIPS: 0.428}{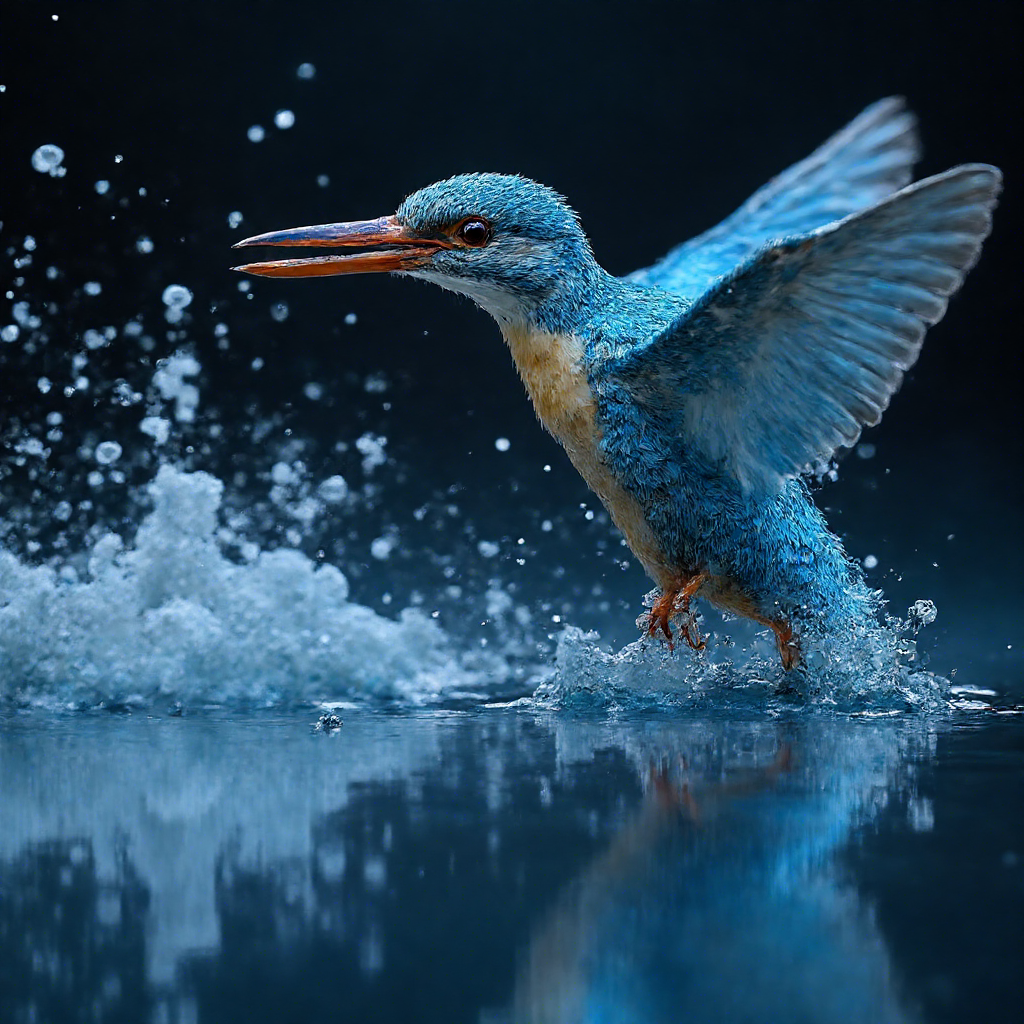} &
    \teaseritem{PermuQuant INT (W3A3)\\LPIPS: 0.327}{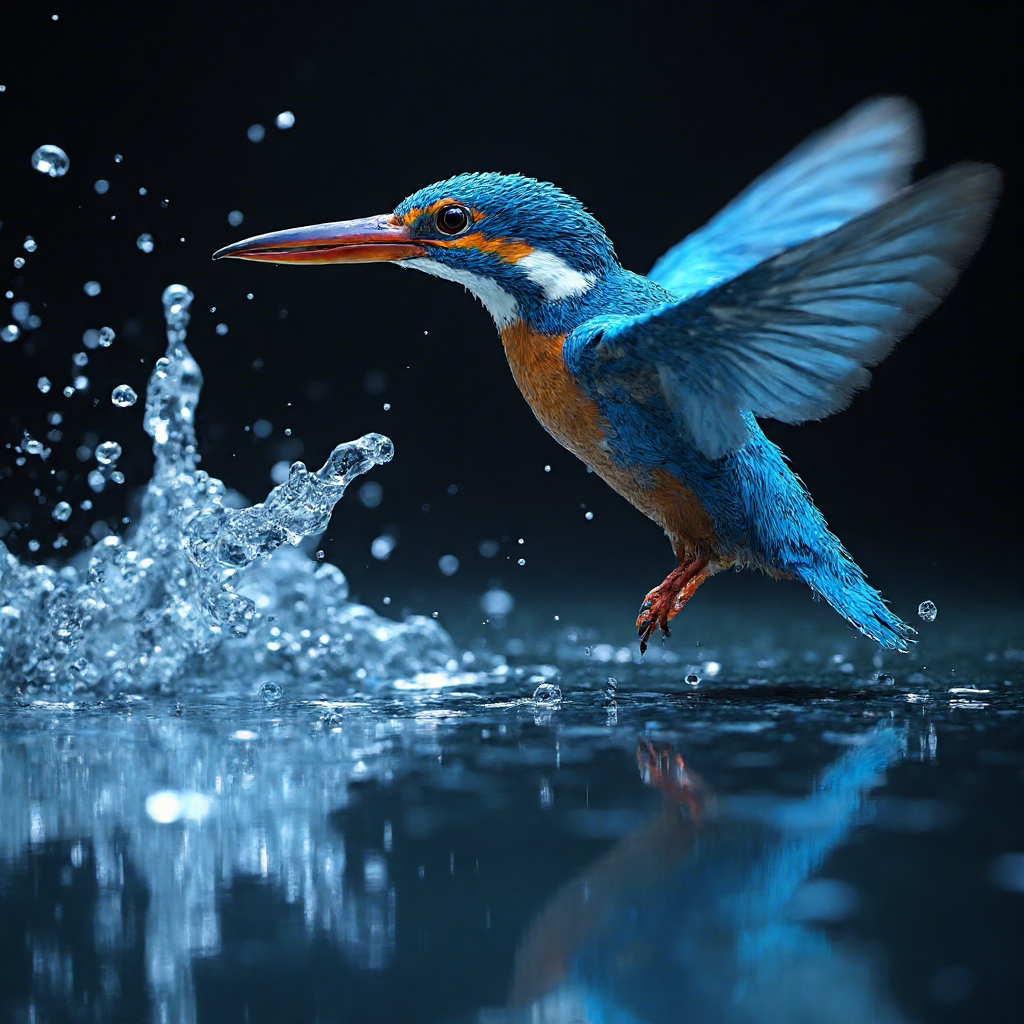}
\end{tabular}

\vspace{-2.8mm}

\teaserprompt{A kingfisher swooping to break digitised binary water, captures information in beak}

% \vspace{-2mm}

\refstepcounter{figure}
\label{fig:teaser}
\begin{minipage}{0.96\textwidth}
\small
Figure~\thefigure:
\textbf{PermuQuant} is a post-training quantization framework for low-bit diffusion models.
In the W4A4 setting, it achieves 3.5$\times$ DiT memory reduction and 6.3$\times$ speedup on a single RTX 5090 32GB GPU by eliminating CPU offloading.
In the challenging W3A3 setting, PermuQuant still produces visually clean results with faithful details, significantly outperforming other baselines.
\end{minipage}

% \vspace{-1mm}
\end{center}

\begin{abstract}
\vspace{-2mm}
  Large-scale visual generative models have achieved remarkable performance. However, their high computational and memory costs make deployment challenging in resource-constrained scenarios, such as interactive applications and personal single-GPU usage. Post-training quantization (PTQ) offers a practical solution by compressing pretrained models without expensive retraining. However, existing PTQ methods still suffer from severe quality degradation under extremely low-bit settings. In this paper, we identify channel ordering as an important but underexplored factor in per-group quantization. In this setting, each contiguous group shares one quantization scale. When channels with very different statistics are placed in the same group, the scale can be dominated by outliers and cause large quantization errors. Based on this observation, we propose \textbf{PermuQuant}, a simple and effective PTQ framework for low-bit diffusion models. PermuQuant sorts channels by a joint second-moment criterion before per-group quantization, placing channels with similar activation and weight statistics into the same group. It further uses a calibration-based acceptance rule to apply reordering only when the selected permutation reduces quantization error on calibration data. The selected permutations are absorbed into adjacent modules or applied to weights offline, avoiding explicit runtime permutation operations. Extensive experiments on multiple large diffusion models show that PermuQuant consistently reduces quantization error and outperforms existing PTQ baselines. On FLUX.1-dev with an RTX 5090, PermuQuant achieves up to a \textbf{1.7}$\times$ single step speedup and reduces the DiT memory footprint by \textbf{3.5}$\times$ under W4A4 NVFP4 quantization. Code will be available at \url{https://github.com/yscheng04/PermuQuant}.
  \vspace{-2mm}
\end{abstract}

\vspace{-2mm}
\section{Introduction}
\label{sec:introduction}

Diffusion models have become a dominant paradigm for visual generation.
By learning to reverse a gradual denoising process, they can generate high-fidelity images with strong diversity~\cite{sohl2015deep,ho2020denoising}.
Recent advances further improve their sampling process, noise schedules, and continuous-time formulations~\cite{song2020denoising,nichol2021improved,song2020score}.
Together with classifier-free guidance and latent diffusion, these methods have enabled powerful text-to-image generation systems~\cite{ho2022classifier,rombach2022high,podell2023sdxl}.
Meanwhile, the backbone of diffusion models has shifted from convolutional UNets to scalable transformer architectures~\cite{ho2020denoising,xie2025sana,cai2025z,flux2024}.
This trend greatly improves generation quality.
However, it also increases the memory footprint and computational cost of inference.
Since diffusion sampling requires repeated network evaluations, the cost becomes especially high for large models and high-resolution generation.

Model quantization~\cite{nagel2021white} has emerged as a widely adopted technique to compress floating-point parameters and activations into lower-precision formats, thereby reducing memory overhead and accelerating inference. The resulting speedup stems from both diminished cache loading costs and faster low-bit arithmetic. Existing quantization approaches are broadly categorized into two paradigms: post-training quantization (PTQ) and quantization-aware training (QAT). While QAT explicitly models quantization effects during training and achieves strong performance at extremely low bit-widths~\cite{qin2020forward, chen2025efficientqat}, it typically requires full or partial model retraining, incurring substantial computational and data costs. In contrast, PTQ operates directly on pretrained checkpoints without additional training, making it considerably more resource-efficient. Recent advances in PTQ~\cite{xiao2023smoothquant, ashkboos2024quarot, li2024svdquant, liu2025clq} have demonstrated that competitive accuracy can be preserved even at low bit-widths. These methods offer a highly practical pathway for deploying large-scale visual models in resource-constrained environments. For example, interactive applications require low-latency inference.

\begin{wrapfigure}{r}{0.40\linewidth}
    \vspace{-3.3mm}
    \centering
    \includegraphics[width=1.0\linewidth]{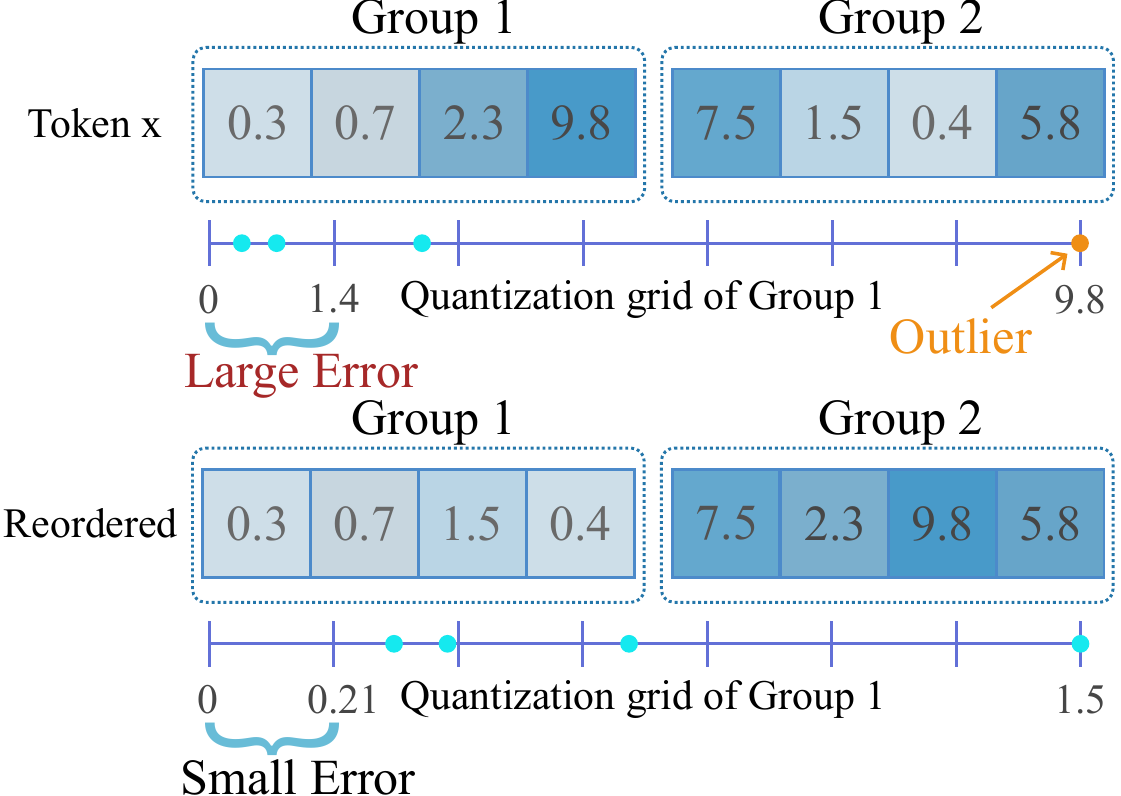}
    \vspace{-3.3mm}
    \caption{
    Example of per-group quantization. Quantization error is greatly affected by channel orders.
    }
    \label{fig:intro_vis}
    \vspace{-4mm}
\end{wrapfigure}

Despite this progress, extremely low-bit weight-activation quantization remains challenging (see Fig.~\ref{fig:teaser}).
In particular, 3-bit quantization can cause severe degradation in image quality and text-image alignment.
A common remedy is per-group quantization, which partitions each token or channel into small contiguous groups of equal size and assigns each group its own scale.
This strategy gives finer-grained scaling than per-token or per-channel quantization.

The \textbf{key insight} of this work is simple: 
\emph{per-group quantization makes channel order important}.
As illustrated in Fig.~\ref{fig:intro_vis}, since groups are formed over contiguous channels, the order decides which channels share a scale.
This matters because each group scale is controlled by the largest-magnitude values in the group.
If statistically mismatched channels are grouped together, outlier channels dominate the shared scale.
The remaining channels are then quantized on an overly coarse grid, which increases quantization error.
By contrast, placing channels with similar statistics into the same group produces more balanced and better matched scales, thereby reducing quantization error.

In this paper, we propose \textbf{PermuQuant}, a simple and effective PTQ framework for low-bit diffusion models. 
PermuQuant exploits the group structure of per-group quantization by reordering channels before quantization. 
It uses channel second moments estimated from calibration data to construct more favorable contiguous groups. 
To formalize this strategy, we show that second-moment based sorting minimizes an upper bound of the expected quantization error. 
We further derive an approximation guarantee for the expected error objective under the uniform noise approximation~\cite{marco2005validity}. 
These results provide a theoretical justification for using channel second moments to guide reordering.

To make reordering practical for linear layers, \textbf{firstly}, PermuQuant jointly considers activation and weight statistics. This is necessary because preserving the function of a linear layer requires applying the same permutation to both sides. We therefore use a joint second-moment criterion to select the channel order.
\textbf{Secondly}, not all layers benefit equally from reordering. To address this issue, we introduce a calibration-based acceptance rule. It applies reordering only when the selected permutation sufficiently reduces the quantization error over calibration data.
\textbf{Thirdly}, another practical challenge is that explicit channel permutation may introduce extra runtime operations. To avoid this cost, we apply the permutation to weight matrix offline and absorb the activation-side permutation into adjacent modules. This design enables PermuQuant to improve low-bit quantization without explicit runtime permutation operations.
Extensive experiments on multiple large diffusion models show that PermuQuant consistently reduces quantization error and delivers strong performance under extremely low-bit settings. We further validate its practical efficiency on Flux.1-dev with an RTX 5090, where W4A4 NVFP4 quantization achieves a \textbf{1.7}$\times$ single step speedup and reduces the DiT memory footprint by \textbf{3.5}$\times$. Our main contributions are summarized as follows:
\begin{itemize}
    \item We are first to identify channel order as an important factor in per-group quantization.
    We show that channel reordering can be used to reduce low-bit quantization error.

    \item We propose PermuQuant, a simple PTQ framework that makes channel reordering practical for quantized linear layers.
    It jointly considers activation and weight statistics, and uses a calibration-based rule to apply reordering only when it provides reliable gains.

    \item Extensive experiments on multiple large diffusion models demonstrate the effectiveness of PermuQuant. Our method consistently outperforms existing PTQ baselines under low-bit settings. On Flux-1.dev with an RTX 5090, PermuQuant achieves a \textbf{1.7}$\times$ single step speedup and reduces the DiT memory by \textbf{3.5}$\times$ under W4A4 NVFP4 quantization.
\end{itemize}

\vspace{-3mm}
\section{Related Work}
\vspace{-2mm}
\subsection{Diffusion Model}
\vspace{-2mm}

Diffusion models have become a dominant family of generative models for high-quality visual generation.
They generate samples by reversing a gradual noising process, where a neural network iteratively denoises corrupted inputs~\cite{sohl2015deep,ho2020denoising}.
Recent advances in sampling, noise schedules, and conditional generation have enabled powerful text-to-image systems~\cite{song2020denoising,nichol2021improved,song2020score,dhariwal2021diffusion,ho2022classifier,rombach2022high,balaji2022ediff,podell2023sdxl}.
Meanwhile, the backbone architecture has gradually shifted from convolutional UNets~\cite{ronneberger2015u,ho2020denoising} to scalable transformer-based designs~\cite{cai2025z,xie2025sana,flux2024}.
Despite their strong generation quality, diffusion models remain expensive at inference time due to sequential denoising steps and increasingly large backbones.
Existing acceleration methods mainly follow two directions.
One direction reduces the number of denoising steps through fast samplers or distillation~\cite{zhang2022fast,liu2022pseudo,lu2022dpm,salimans2022progressive,meng2023distillation,song2023consistency,luo2023latent,sauer2024adversarial,yin2024one,yin2024improved}.
The other reduces the cost of each step through efficient architectures, sparse inference, distributed inference, and quantization~\cite{unlu2020efficient,li2023snapfusion,lyu2022accelerating,ma2024deepcache,liu2025reusing,li2024distrifusion,fang2024pipefusion,chen2024asyncdiff,shang2023post,li2023q}.
This work focuses on low-bit quantization, which reduces memory and computation without changing the sampling schedule.
\vspace{-2mm}
\subsection{Model Quantization}
\vspace{-2mm}
Model quantization~\cite{nagel2021white} is a widely used technique for compressing model size and accelerating model inference.
By representing weights and activations with low-bit numerical formats, quantization can substantially reduce memory footprint and computational cost.
Existing methods mainly include quantization-aware training (QAT) and post-training quantization (PTQ).
QAT~\cite{chen2025efficientqat,qin2023quantsr,liu2020reactnet,martinez2020training,qin2020forward} simulates low-precision operations during training and performs well in extremely low-bit settings, but requires costly retraining.
PTQ~\cite{liu2021post, liu2025low}, in contrast, directly quantizes pretrained models without end-to-end retraining and is more resource-efficient.
Recent PTQ methods have enabled 4-bit quantization for LLMs and diffusion models with limited quality degradation~\cite{liu2024spinquant,li2024svdquant}.
GPTQ~\cite{frantar2022gptq} uses Hessian-based compensation, while SmoothQuant~\cite{xiao2023smoothquant}, OmniQuant~\cite{OmniQuant}, QuaRot~\cite{ashkboos2024quarot}, and ConvRot~\cite{huang2025convrot} suppress activation outliers through smoothing or orthogonal transformations.
SVDQuant~\cite{li2024svdquant} absorbs outliers into a low-rank branch.
In this work, we focus on per-group quantization and demonstrate channel reordering can further reduce quantization error, even under 3-bit quantization.

\begin{figure}[t]
    \centering
    \captionsetup[subfigure]{skip=0pt}
    \begin{subfigure}[t]{0.24\linewidth}
        \centering
        \includegraphics[width=\linewidth]{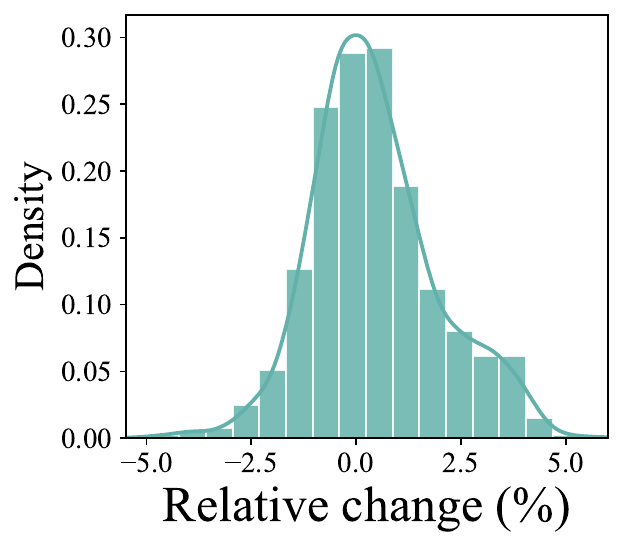}
        \caption{\hspace{1mm}Random permutation}
        \label{fig:random_perm}
    \end{subfigure}
    \begin{subfigure}[t]{0.24\linewidth}
        \centering
        \includegraphics[width=\linewidth]{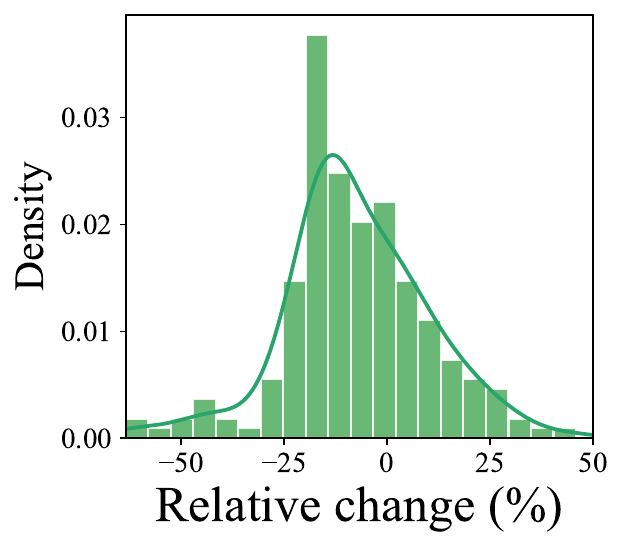}
        \caption{Sorting}
        \label{fig:sm_sort}
    \end{subfigure}
    \hfill
    \raisebox{0.5mm}{
        \begin{tikzpicture}
            \draw[dashed, line width=0.8pt] (0,0) -- (0,2.8cm);
        \end{tikzpicture}
    }
    \hfill
    \begin{subfigure}[t]{0.24\linewidth}
        \centering
        \raisebox{0mm}{%
            \includegraphics[width=\linewidth]{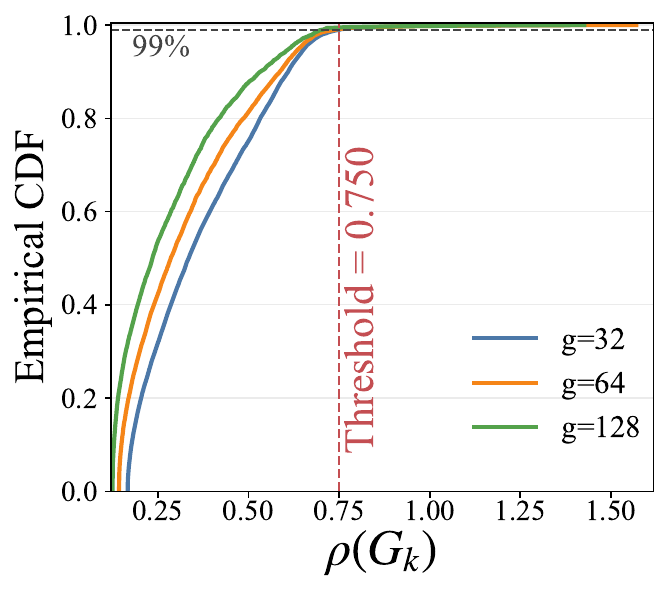}
        }
        \caption{Activation ratio}
        \label{fig:ratio_activation}
    \end{subfigure}
    \hfill
    \begin{subfigure}[t]{0.24\linewidth}
        \centering
        \raisebox{0mm}{
            \includegraphics[width=\linewidth]{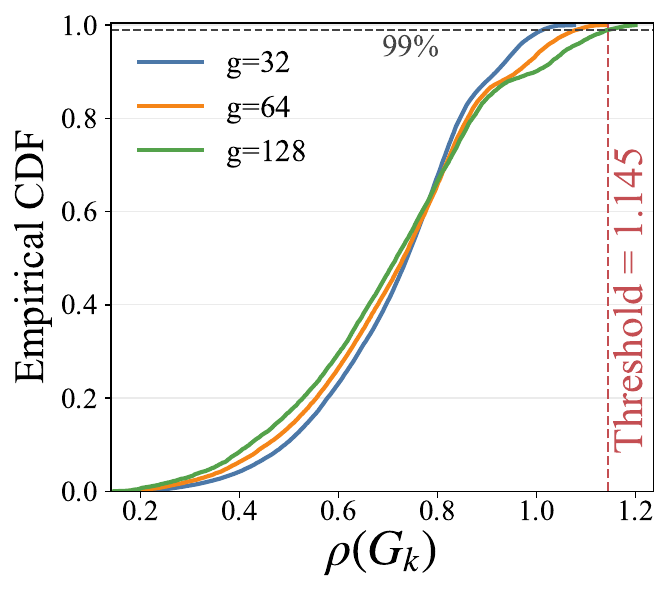}
        }
        \caption{Weight ratio}
        \label{fig:ratio_weight}
    \end{subfigure}
    \hfill

    \vspace{-2mm}
    \caption{
    (a) and (b): Relative change in activation quantization error caused by random channel permutations and sorting channels by their second moments on Z-Image-Turbo.
    (c) and (d): Empirical validation of Assumption~\ref{assump:extremal_control}.
    We plot the CDF of the ratio $\rho(G_k)$, which estimates the constant $C$ required by each group.
    For more than 99\% of groups in activation and weight, $\rho(G_k)<1.2$.
    }
    \vspace{-4.5mm}
    \label{fig:reorder_visualization}
\end{figure}

\vspace{-3mm}
\section{Method}
\vspace{-3mm}

In this section, we first review the basic concept of quantization in Sec.~\ref{sec:background}, with a focus on per-group quantization. 
We then present the channel reordering technique in Sec.~\ref{sec:method-reorder}, which aims to reduce per-group quantization error. We first motivate why channel order matters, then describe how channels are reordered, with theoretical analysis to explain its effectiveness.
Finally, we propose PermuQuant in Sec.~\ref{sec:permuquant}. We extend channel reordering to practical quantized linear layers with joint second-moment channel reordering. We then adopt calibration-based acceptance rule to select effective permutations. We also discuss how the selected permutations can be deployed without explicit runtime operations.

\vspace{-2mm}
\subsection{Background}
\vspace{-2mm}
\label{sec:background}

\textbf{Quantization} reduces the memory cost and inference latency of
large models by representing weights or activations with low-bit integers.
In low-bit settings, a scale factor maps floating-point values to a discrete
integer grid, so the quantization error strongly depends on the scales. A common practice is to use fine-grained scaling, such as per-token
activation quantization or per-channel weight quantization, since different
tokens or channels can have very different value ranges.

In extremely low-bit settings, an even more fine-grained scheme, namely
\emph{per-group quantization}, is often adopted. It further partitions the
values within each token or channel into contiguous groups of equal size.
All scalars within the same group share one quantization scale. Let
$\mathbf{x}\in\mathbb{R}^{d}$ denote a token-wise or channel-wise vector.
We divide its entries into $K=d/g$ contiguous groups
$G_1,\ldots,G_K$, where each group has size $g$. For a group $G_k$, the
symmetric quantization scale is computed as
\begin{equation}
    s_k = \frac{\max_{i\in G_k} |x_i|}{Q},
\end{equation}
where $Q$ denotes the largest representable integer. Each element
$x_i$ is quantized and dequantized as
\begin{equation}
    z_i =
    \mathrm{clip}\!\left(
    \mathrm{round}\!\left(\frac{x_i}{s_k}\right),
    -Q, Q
    \right),
    \qquad
    \hat{x}_i = s_k z_i .
\end{equation}

\vspace{-4mm}
\subsection{Lower Quantization Error with Reordering}
\vspace{-2mm}
\label{sec:method-reorder}
For clarity, we present the following analysis using activation quantization as an example. The same idea also applies to weight quantization by treating weight channels as the grouped vectors. 

\textbf{Motivation.} In per-group quantization, channels are partitioned into contiguous groups, and all channels within the same group share one quantization scale. As a result, the quantization error depends not only on the underlying value distribution, but also on the channel order. Therefore, even when the set of channels remains unchanged, reordering them can lead to substantially different quantization errors. We also empirically observe that different random permutations can produce noticeably different quantization errors in Fig.~\ref{fig:random_perm}. It highlights that the permutation is an important factor in per-group quantization. Therefore, the core issue is how to find a good permutation.
More formally, let $\pi$ denote a permutation over the $d$ channels, and let 
$\{G_k(\pi)\}$ be the contiguous groups induced by $\pi$. 
Our goal is to find a permutation that minimizes the expected quantization error:
\begin{equation}
    \pi^{\star}
    =
    \arg\min_{\pi}
    \mathbb{E}_{\mathbf{x}}
    \left[
    \left\|
    \mathbf{x}
    -
    \hat{\mathbf{x}}_{\pi}
    \right\|_2^2
    \right],
    \label{eq:perm_objective}
\end{equation}
where $\hat{\mathbf{x}}_{\pi}$ denotes the dequantized tensor obtained under the grouping $\{G_k(\pi)\}$.

\begin{figure}[t]
    \centering
    \includegraphics[width=1.0\linewidth]{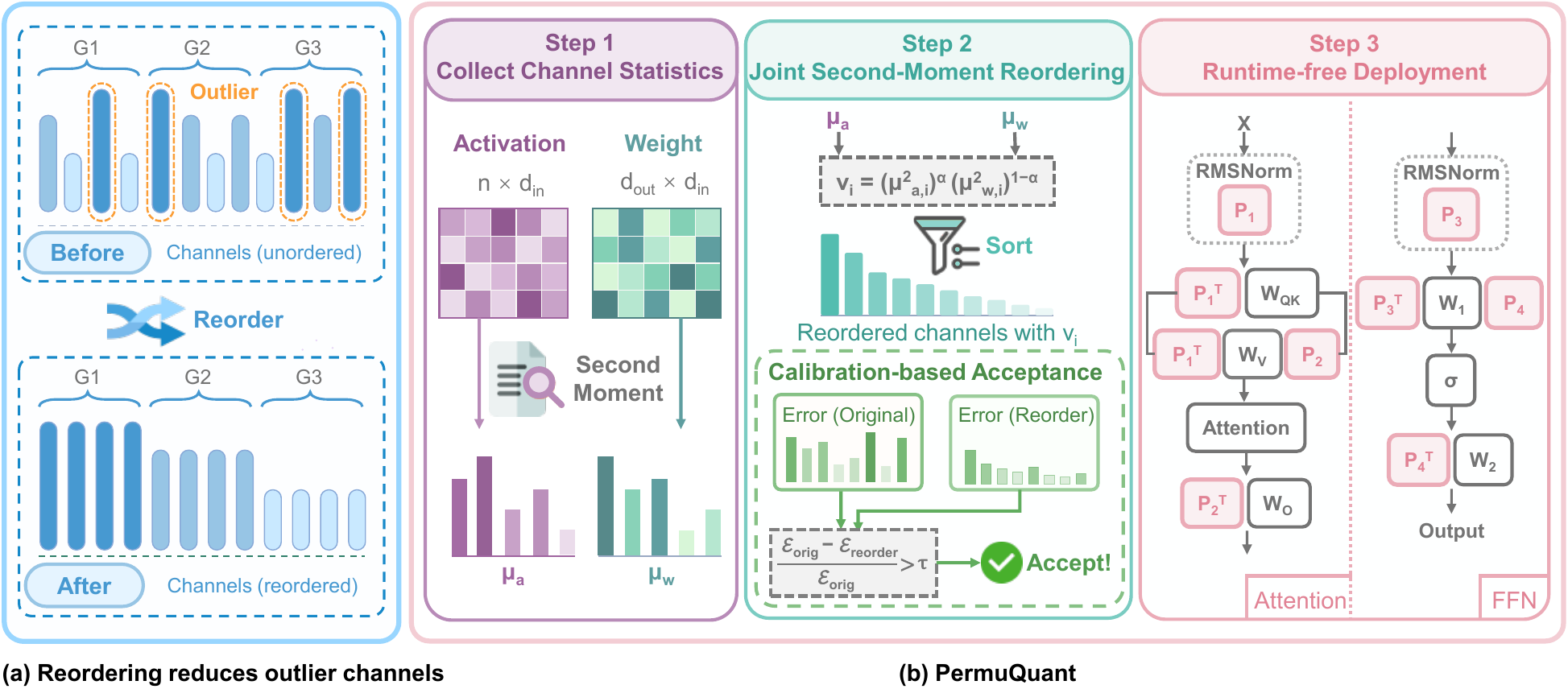}
    \vspace{-5.5mm}
    \caption{
        Overview of PermuQuant.
        (a) Channel reordering places channels with similar statistics into the same group.
        (b) PermuQuant first collects activation and weight statistics from calibration data.
        It then performs joint second-moment reordering with calibration-based acceptance.
        The accepted permutations are folded into adjacent modules offline for runtime-free deployment.
    }
    \label{fig:permuquant}
    \vspace{-4.5mm}
\end{figure} 

\textbf{Intuition.}
The objective in Eq.~\eqref{eq:perm_objective} is difficult to optimize directly, since the permutation space is combinatorial. 
We therefore seek a simple criterion that can identify low-error group partitions. 
It's well known that in low-bit quantization, outliers are a major source of error. 
Under per-group quantization, however, the effect of an outlier is determined by its group. 
A high-magnitude channel may dominate the scale when grouped with low-magnitude channels, but becomes less harmful when grouped with channels of similar magnitudes. 
Therefore, a desirable permutation should avoid mixing channels with very different value ranges in the same group.
Based on this observation, the intuition of our method is simple: \textbf{channels with similar magnitude should be placed into the same group.} 

\textbf{Second-moment based reordering.}
To formalize this intuition, we derive an upper bound on the expected quantization error in Eq.~\eqref{eq:perm_objective}. We then show how this bound can be reduced through an appropriate channel reordering.
Consider a random token vector $\mathbf{x} \in \mathbb{R}^{d}$ drawn from the activation distribution, and let $\hat{\mathbf{x}}$ be its dequantized version. 
Denote by $\mathbf{e} = \mathbf{x} - \hat{\mathbf{x}}$ the quantization error. 
For a fixed grouping $\{G_k\}$ with group size $g$, the scale of each group is determined by its largest-magnitude entry. 
Therefore, the expected quantization error is bounded by
\begin{equation}
    \mathbb{E}_{\mathbf{x}}\|\mathbf{e}\|_2^2
    \le
    \frac{g}{4Q^2}
    \sum_k
    \mathbb{E}_{\mathbf{x}}\left[
    \max_{i \in G_k} |x_i|^2
    \right],
    \label{eq:exp_error_pre}
\end{equation}
where $Q$ is the maximum quantization level. See Appendix~\ref{app:proof_eq6} for the proof.

Equation~\eqref{eq:exp_error_pre} shows that the expected quantization error depends on the expected squared maximum within each group. However, directly optimizing this term is difficult, since it depends on the joint distribution of channels assigned to the same group and changes with every candidate permutation. To relate this quantity to channel-wise statistics, we propose the following assumption. 

\begin{assumption}[Extremal control]
\label{assump:extremal_control}
\vspace{-2mm}
For each channel $i$, let $\mathbb{E}[x_i^2] = \mu_i^2$. We assume that there exists a constant $C > 0$ such that, for any group $G_k$ of size $g$,
\begin{equation}
    \mathbb{E}\left[
    \max_{i \in G_k} |x_i|^2
    \right]
    \le
    C \log_2(2g)\,
    \max_{i \in G_k} \mu_i^2.
    \label{eq:extremal_control}
\end{equation}
\vspace{-2mm}
\end{assumption}

Intuitively, Assumption~\ref{assump:extremal_control} provides a bound on the expected squared maximum within a group. To empirically validate Assumption~\ref{assump:extremal_control}, we compute the group-wise ratio
$\rho(G_k)=
\mathbb{E}[\max_{i\in G_k}|x_i|^2]/
(\log_2(2g)\max_{i\in G_k}\mu_i^2)$,
which estimates the smallest constant $C$ required for each group.
We then plot the empirical CDF of $\rho(G_k)$ over all groups and group sizes.
As shown in Fig.~\ref{fig:reorder_visualization}, more than $99\%$ of groups have $\rho(G_k)<1.2$, indicating that the extremal-control condition holds with a small empirical constant for most groups. Substituting Eq.~\eqref{eq:extremal_control} into Eq.~\eqref{eq:exp_error_pre}, we obtain

\vspace{-2mm}
\begin{equation}
    \mathbb{E}\|\mathbf{e}\|_2^2
    \le
    \frac{C g \log(2g)}{4Q^2}
    \sum_k
    \max_{i \in G_k} \mu_i^2.
    \label{eq:exp_error_var_bound}
\end{equation}

Equation~\eqref{eq:exp_error_var_bound} shows that the expected quantization error is controlled by the group-wise maxima of channel second moments. 
Therefore, the upper bound can be reduced by placing channels with similar second moments into the same group. Actually, we have the following proposition:
\begin{proposition}[Optimality of second-moment sorting]
\label{prop:second_moment_sorting}
For a fixed group size $g$, sorting channels in descending order of $\mu_i^2$ and then partitioning them into contiguous groups minimizes $\sum_k \max_{i \in G_k} \mu_i^2$ among all contiguous group partitions induced by channel permutations.
\end{proposition}
We defer the proof to Appendix~\ref{app:proof_optimal_bound}. This motivates our second-moment based channel reordering strategy, where \textbf{channels are sorted by their second moments}.

Beyond minimizing this upper bound, the second-moment based ordering also provides an approximation guarantee for the expected quantization error objective under the uniform-noise approximation~\cite{marco2005validity}. The following theorem states the resulting approximation ratio.

\begin{theorem}[Approximation guarantee]
\label{thm:second_moment_sorting}
Let $\mathcal{E}(P):=\mathbb{E}_{x}\mathbb{E}_{\epsilon\mid x}\|\epsilon\|_2^2$ denote the expected quantization error under the uniform-noise approximation~\cite{marco2005validity}. $\epsilon$ is the additive quantization noise and $\epsilon_i\mid x \sim \mathrm{Unif}(-s_k/2,s_k/2)$ for $i\in G_k$. Let $P_{\mathrm{sort}}$ be the partition obtained by sorting channels in descending order of $\mu_i^2$, where $\mathbb{E}[x_i^2] = \mu_i^2$. Assume that Assumption~\ref{assump:extremal_control} holds. Then
\begin{equation}
    \mathcal{E}(P_{\mathrm{sort}})
    \le
    C \log_2(2g)\cdot \min_P \mathcal{E}(P).
\end{equation}
\end{theorem}
Therefore, second-moment based sorting is a $C\log_2(2g)$-approximation to the optimal partition with respect to expected quantization error. We defer the full proof to Appendix~\ref{app:proof_permuquant}.

\vspace{-1mm}
\subsection{PermuQuant}
\vspace{-1mm}
\label{sec:permuquant}
\textbf{Joint second-moment channel reordering.} 
The above analysis considers only one side of the linear layer. In practice, reordering must preserve the function of the linear layer. This requires applying the same permutation to both the activations and the corresponding input channels of the weight matrix. 
Given an activation matrix $\mathbf{X}$, a weight matrix $\mathbf{W}$, and a permutation matrix $\mathbf{P}$, we have
\begin{equation}
    \mathbf{Y}=\mathbf{X}\mathbf{W}
    =\mathbf{X}\mathbf{P}\mathbf{P}^{\top}\mathbf{W}
    \approx Q(\mathbf{X}\mathbf{P})Q(\mathbf{P}^{\top}\mathbf{W}),
    \label{eq:reordered_quant_linear}
\end{equation}
where $Q(\cdot)$ denotes the quantize-then-dequantize operator. 
Therefore, the final channel order should take into account the statistics of both activation and weight sides.

For each input channel $i$, let $\mu_{a,i}^2$ denote the activation second moment and $\mu_{w,i}^2$ denote the second moment of the corresponding weight channel. We define a joint second moment criterion as
% \vspace{-2mm}
\begin{equation}
    v_i = (\mu_{a,i}^2)^{\alpha} (\mu_{w,i}^2)^{1-\alpha},
    \label{eq:joint_second_moment_rank}
\end{equation}
where $\alpha \in [0,1]$ balances the relative importance of activation and weight. We then sort channels according to $v_i$. Different layers may be dominated by activation or weight to different extents. To account for this, we uniformly enumerate a set of candidate values for $\alpha$ in the interval $[0,1]$. We then select the value that yields the smallest quantization error on the calibration data.

\textbf{Calibration-based acceptance rule.}
However, reordering is not uniformly beneficial for all layers. 
First, activations and weights may prefer different channel orders. Thus, a permutation that reduces error on one side can increase the error on the other. Second, some layers may not have a clear sortable structure in their channel statistics, or the original channel order already provides an effective grouping. In such cases, the overall benefit of reordering can be marginal or even negative. To prevent such cases, we propose a calibration-based acceptance rule. Let $\mathcal{E}_{\mathrm{orig}}$ and $\mathcal{E}_{\mathrm{reorder}}$ denote the quantization errors on the calibration set before and after reordering. We accept the permutation if
\begin{equation}
    \frac{\mathcal{E}_{\mathrm{orig}} - \mathcal{E}_{\mathrm{reorder}}}{\mathcal{E}_{\mathrm{orig}}} > \tau,
\end{equation}
where $\tau$ is a predefined threshold. Otherwise, we keep the original channel order. In this way, PermuQuant is applied only to layers where reordering yields a sufficiently large empirical improvement. 
This avoids possible increase in quantization error when reordering is not effective.

\textbf{Runtime-free deployment.} Importantly, the permutation introduces negligible additional runtime cost. The permutation on weight matrix can be performed offline. The permutation on activation can be absorbed either into the preceding linear modules or into its preceding LayerNorm/RMSNorm.

\begin{table*}[t]
\centering
\caption{
Quantitative comparison on MJHQ30K and GenEval.
For MJHQ30K, we report FID, KID$\times 10^3$, and ImageReward (IR); for GenEval, we report the standard scores. The best and second-best results among \textbf{W3A3} methods are highlighted in \textbf{bold} and \underline{underline}, respectively.
}
\label{tab:main_results}
\small
\setlength{\tabcolsep}{2pt}
\renewcommand{\arraystretch}{1.05}
\newcommand{\secondbest}[1]{\underline{#1}}

\resizebox{\textwidth}{!}{
\begin{tabular}{c@{\hspace{8pt}}ccrrrrrrrrrr}
\toprule
\multirow{2}{*}{Model}
& \multirow{2}{*}{Method}
& \multirow{2}{*}{Venue}
& \multicolumn{3}{c}{MJHQ30K}
& \multicolumn{7}{c}{GenEval} \\
\cmidrule(lr){4-6} \cmidrule(lr){7-13}
& &
& FID $\downarrow$
& KID $\downarrow$
& IR $\uparrow$
& Single $\uparrow$
& Two $\uparrow$
& Count $\uparrow$
& Color $\uparrow$
& Pos. $\uparrow$
& Attr. $\uparrow$
& Overall $\uparrow$ \\
\midrule

\multirow{8}{1.2cm}{\centering \makecell{Z-Image\\-Turbo\\(8 Steps)}}
& BF16 & --
& 67.6 & 5.7 & 1.001
& 1.00 & 0.97 & 0.71 & 1.00 & 0.63 & 0.83 & 0.86 \\
\cmidrule(l){2-13}
 & RTN & --
& 250.8 & 183.0 & -1.609
& 0.69 & 0.21 & 0.21 & 0.47 & 0.21 & 0.10 & 0.32 \\

& SmoothQuant~\cite{xiao2023smoothquant} & \texttt{ICML'23}
& 221.7 & 136.1 & -1.388
& 0.85 & 0.55 & 0.43 & 0.77 & 0.18 & 0.20 & 0.50 \\

& OmniQuant~\cite{OmniQuant} & \texttt{ICLR'24}
& 87.7 & 18.0 & 0.575
& 0.92 & \secondbest{0.87} & \textbf{0.71} & 0.90 & 0.58 & 0.65 & 0.77 \\

& SVDQuant~\cite{li2024svdquant} & \texttt{ICLR'25}
& 186.7 & 92.9 & -0.940
& 0.85 & 0.61 & 0.32 & 0.83 & 0.50 & 0.43 & 0.59 \\

& ConvRot~\cite{huang2025convrot} & \texttt{arXiv'25}
& \secondbest{86.5} & \secondbest{14.1} & \secondbest{0.587}
& \textbf{1.00} & \secondbest{0.87} & \textbf{0.71} & \secondbest{0.97} & \secondbest{0.61} & \secondbest{0.70} & \secondbest{0.81} \\

\rowcolor{lightblue}
\cellcolor{white}
& PermuQuant & --
& 134.9 & 49.2 & -0.163
& \secondbest{0.96} & \secondbest{0.87} & 0.61 & 0.93 & 0.55 & 0.50 & 0.74 \\

\rowcolor{lightblue}
\cellcolor{white}
& PermuQuant-H & --
& \textbf{80.4} & \textbf{12.3} & \textbf{0.715}
& \secondbest{0.96} & \textbf{0.97} & \secondbest{0.64} & \textbf{1.00} & \textbf{0.63} & \textbf{0.75} & \textbf{0.83} \\

\midrule

\multirow{8}{1.2cm}{\centering \makecell{FLUX.1\\-dev\\(50 Steps)}}
& BF16 & --
& 60.6 & 2.8 & 0.917
& 1.0 & 0.89 & 0.61 & 0.93 & 0.68 & 0.53 & 0.77 \\
\cmidrule(l){2-13}
& RTN & --
& 131.0 & 52.4 & 0.386
& 0.81 & 0.45 & 0.29 & 0.67 & 0.32 & 0.10 & 0.44 \\

& SmoothQuant~\cite{xiao2023smoothquant} & \texttt{ICML'23}
& 142.6 & 59.7 & 0.149
& 0.81 & 0.32 & 0.29 & 0.63 & 0.34 & 0.13 & 0.42 \\

& OmniQuant~\cite{OmniQuant} & \texttt{ICLR'24}
& 85.5 & 16.0 & \textbf{0.798}
& 0.92 & 0.71 & 0.46 & 0.83 & 0.55 & 0.38 & 0.64 \\

& SVDQuant~\cite{li2024svdquant} & \texttt{ICLR'25}
& 103.3 & 24.7 & 0.501
& 0.86 & 0.29 & 0.45 & 0.61 & 0.47 & 0.11 & 0.47 \\

& ConvRot~\cite{huang2025convrot} & \texttt{arXiv'25}
& \secondbest{68.6} & \secondbest{6.2} & 0.709
& 0.92 & \secondbest{0.89} & \textbf{0.71} & 0.90 & \secondbest{0.68} & \textbf{0.50} & \textbf{0.77} \\

\rowcolor{lightblue}
\cellcolor{white}
 & PermuQuant & --
& 74.9 & 9.1 & 0.726
& \secondbest{0.96} & 0.84 & 0.50 & \textbf{0.97} & 0.58 & 0.40 & \secondbest{0.71} \\

\rowcolor{lightblue}
\cellcolor{white}
 & PermuQuant-H & --
& \textbf{64.7} & \textbf{4.6} & \secondbest{0.759}
& \textbf{1.00} & \textbf{0.92} & \secondbest{0.54} & \secondbest{0.93} & \textbf{0.74} & \secondbest{0.48} & \textbf{0.77} \\

\midrule

\multirow{8}{1.2cm}{\centering \makecell{SANA-1.5\\-1.6B\\(20 Steps)}}
& BF16  & --
& 58.1 & 1.2 & 1.075
& 1.00 & 0.97 & 0.79 & 0.97 & 0.66 & 0.55 & 0.82 \\
\cmidrule(l){2-13}
 & RTN & --
& 103.6 & 30.1 & 0.394
& \textbf{0.96} & 0.63 & 0.54 & 0.83 & 0.47 & 0.43 & 0.64 \\

 & SmoothQuant~\cite{xiao2023smoothquant} & \texttt{ICML'23}
& 102.6 & 29.4 & 0.466
& 0.88 & 0.68 & 0.46 & 0.87 & 0.47 & 0.45 & 0.64 \\

& OmniQuant~\cite{OmniQuant} & \texttt{ICLR'24}
& 96.8 & 26.0 & 0.541
& \textbf{0.96} & \secondbest{0.76} & \secondbest{0.57} & 0.83 & 0.47 & 0.53 & 0.69 \\

& SVDQuant~\cite{li2024svdquant} & \texttt{ICLR'25}
& 89.7 & 19.5 & 0.513
& \textbf{0.96} & \secondbest{0.76} & 0.46 & \secondbest{0.93} & 0.45 & 0.55 & 0.69 \\

& ConvRot~\cite{huang2025convrot} & \texttt{arXiv'25}
& 98.1 & 24.8 & 0.440
& \textbf{0.96} & 0.74 & 0.46 & 0.83 & \secondbest{0.55} & 0.55 & 0.68 \\

\rowcolor{lightblue}
\cellcolor{white}
 & PermuQuant & --
& \secondbest{82.8} & \secondbest{14.7} & \secondbest{0.700}
& \secondbest{0.92} & \secondbest{0.76} & \textbf{0.61} & \secondbest{0.93} & 0.45 & \secondbest{0.60} & \secondbest{0.71} \\

\rowcolor{lightblue}
\cellcolor{white}
& PermuQuant-H & --
& \textbf{68.0} & \textbf{5.5} & \textbf{0.919}
& \textbf{0.96} & \textbf{0.89} & \secondbest{0.57} & \textbf{0.97} & \textbf{0.66} & \textbf{0.73} & \textbf{0.80} \\

\bottomrule
\end{tabular}
}
\vspace{-3mm}
\end{table*}

\section{Experiments}
\vspace{-2mm}
\subsection{Experimental Settings}
\vspace{-2mm}

\textbf{Data and evaluation.}
Following prior work~\cite{li2023q}, we sample calibration prompts from COCO Captions~\cite{lin2014microsoft}.
For generation quality, we sample 1K prompts from MJHQ-30K~\cite{li2024playground} and report FID~\cite{heusel2017gans, parmar2022aliased}, KID~\cite{binkowski2018demystifying}, and ImageReward~\cite{xu2023imagereward}.
For instruction alignment, we use GenEval~\cite{ghosh2023geneval}, covering single object, two objects, counting, color, position, and attribute binding.
Please refer to Appendix~\ref{app:implementation_detail} for more implementation details.

\begin{figure*}[t]
\centering
\scriptsize
\setlength{\tabcolsep}{2pt}
\renewcommand{\arraystretch}{0.95}

\newlength{\qualimgsize}
\setlength{\qualimgsize}{0.16\textwidth}

\newcommand{\imgvsep}{%
  \begin{tikzpicture}[baseline={(0,0)}]
    \draw[densely dashed, line width=0.4pt] (0,0) -- (0,\qualimgsize);
  \end{tikzpicture}%
}

\newcommand{\promptline}[1]{%
\multicolumn{7}{c}{%
\parbox{0.7\textwidth}{%
\centering
\setlength{\baselineskip}{0.85\baselineskip}%
\textit{Prompt: #1}%
}%
}%
}

\begin{adjustbox}{width=\textwidth}
\begin{tabular}{c@{\hspace{0.5mm}}c@{\hspace{1.5mm}}ccccc}

\textbf{SANA-1.5-1.6B} &
&
\textbf{OmniQuant} &
\textbf{SVDQuant} &
\textbf{ConvRot} &
\textbf{PermuQuant} &
\textbf{PermuQuant-H} \\

ImageReward: 1.753 &
&
ImageReward: 1.410 &
ImageReward: 1.676 &
ImageReward: 1.494 &
ImageReward: 1.741 &
ImageReward: 1.792 \\[0.2em]

\includegraphics[width=\qualimgsize]{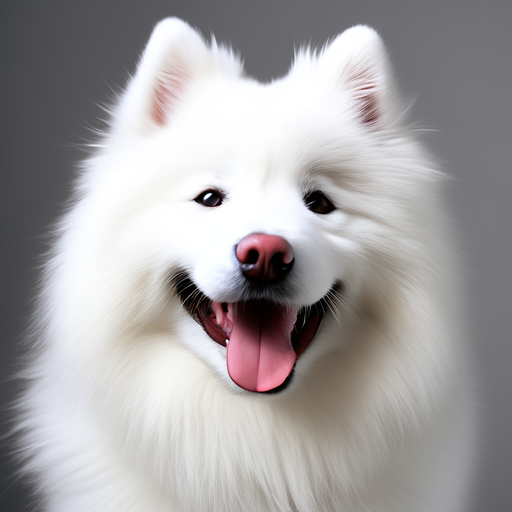} &
\imgvsep &
\includegraphics[width=\qualimgsize]{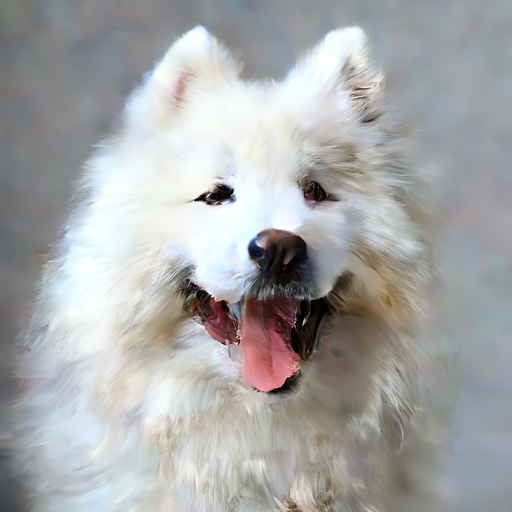} &
\includegraphics[width=\qualimgsize]{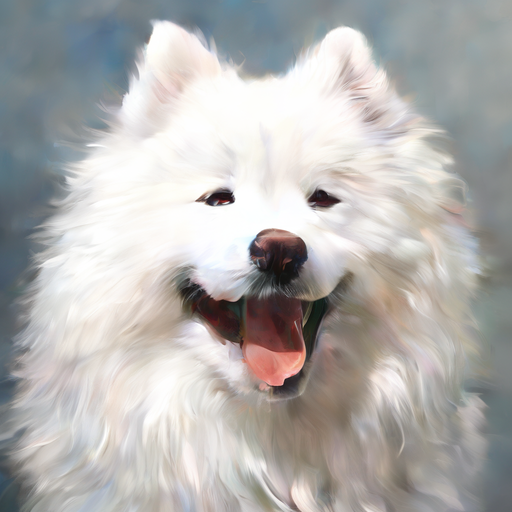} &
\includegraphics[width=\qualimgsize]{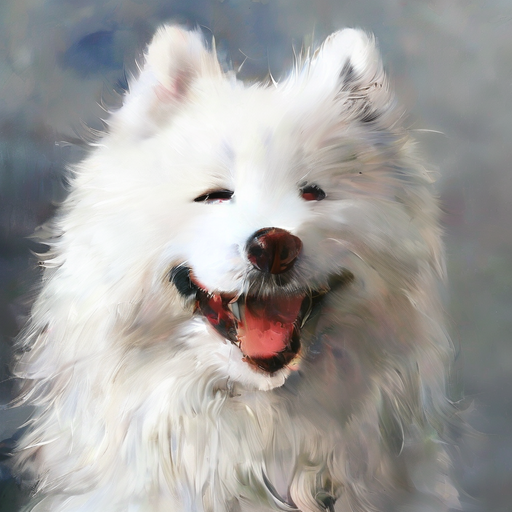} &
\includegraphics[width=\qualimgsize]{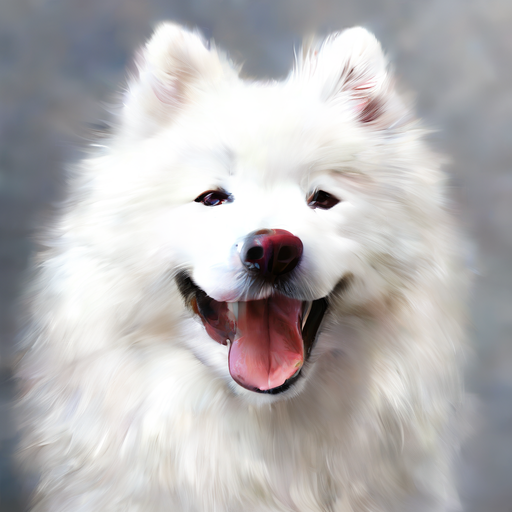} &
\includegraphics[width=\qualimgsize]{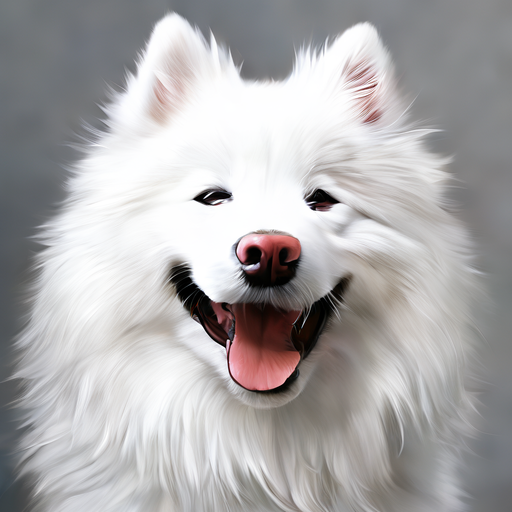} \\[-0.2em]

\promptline{happy white Samoyed}\\[0.0mm]

ImageReward: 1.576 &
&
ImageReward: 1.004 &
ImageReward: 0.908 &
ImageReward: 0.562 &
ImageReward: 1.257 &
ImageReward: 1.538 \\[0.2em]

\includegraphics[width=\qualimgsize]{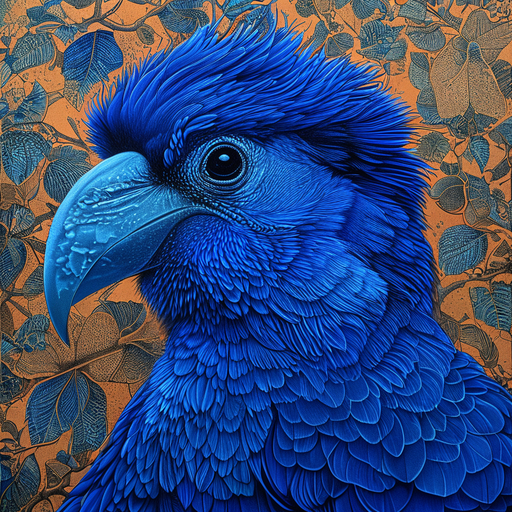} &
\imgvsep &
\includegraphics[width=\qualimgsize]{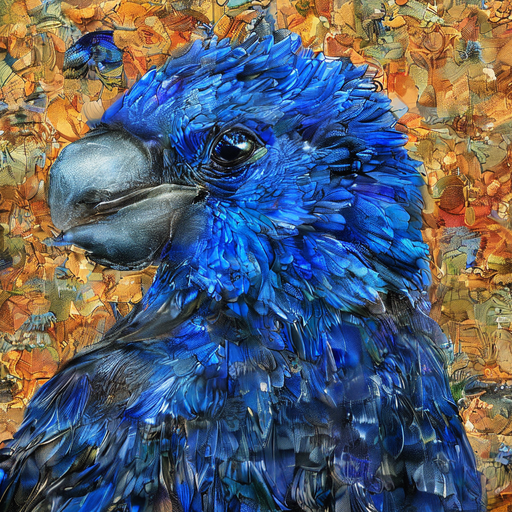} &
\includegraphics[width=\qualimgsize]{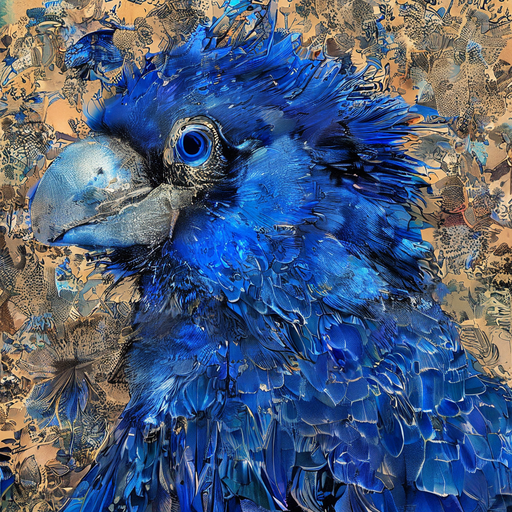} &
\includegraphics[width=\qualimgsize]{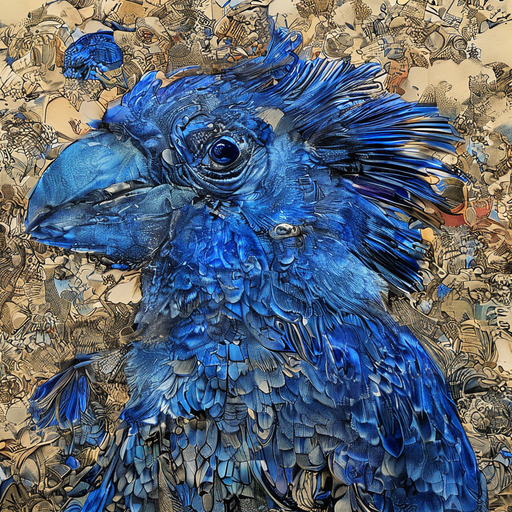} &
\includegraphics[width=\qualimgsize]{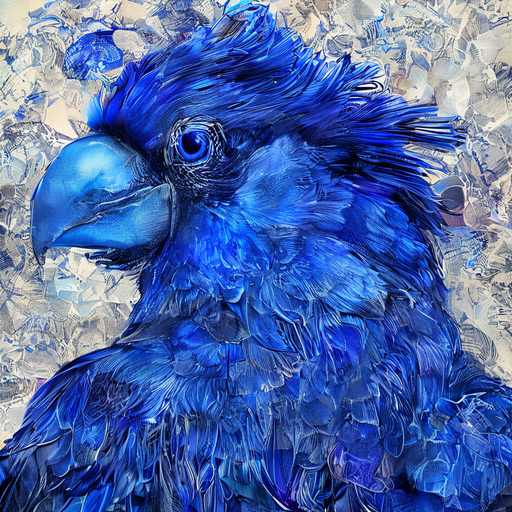} &
\includegraphics[width=\qualimgsize]{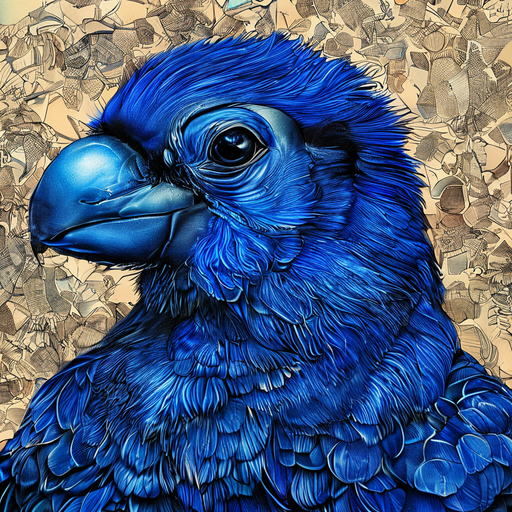} \\[-0.2em]

\promptline{a portrait of a blue coua inspired by Ogawa Kazumasa, insanely detailed and intricate}\\[0.2mm]

\multicolumn{7}{c}{\textbf{(a) SANA-1.5-1.6B}} \\[0.4mm]

\textbf{FLUX.1-dev} &
&
\textbf{OmniQuant} &
\textbf{SVDQuant} &
\textbf{ConvRot} &
\textbf{PermuQuant} &
\textbf{PermuQuant-H} \\

ImageReward: 1.516 &
&
ImageReward: 1.683 &
ImageReward: 0.723 &
ImageReward: 1.167 &
ImageReward: 1.436 &
ImageReward: 1.560 \\[0.2em]

\includegraphics[width=\qualimgsize]{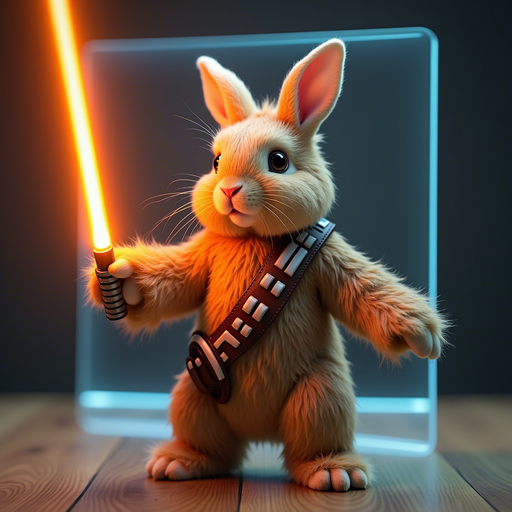} &
\imgvsep &
\includegraphics[width=\qualimgsize]{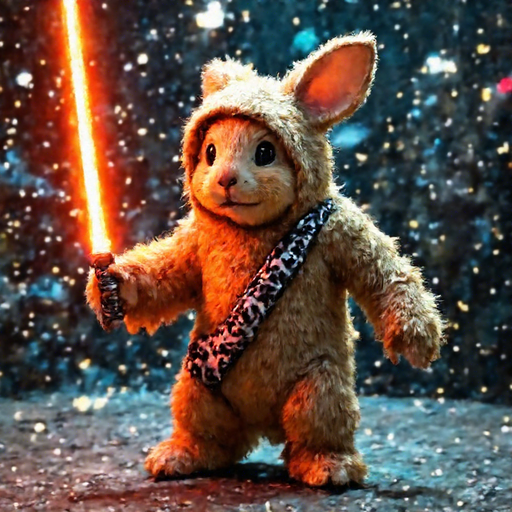} &
\includegraphics[width=\qualimgsize]{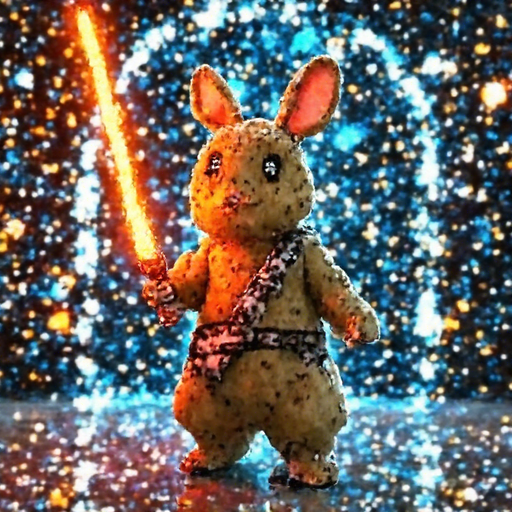} &
\includegraphics[width=\qualimgsize]{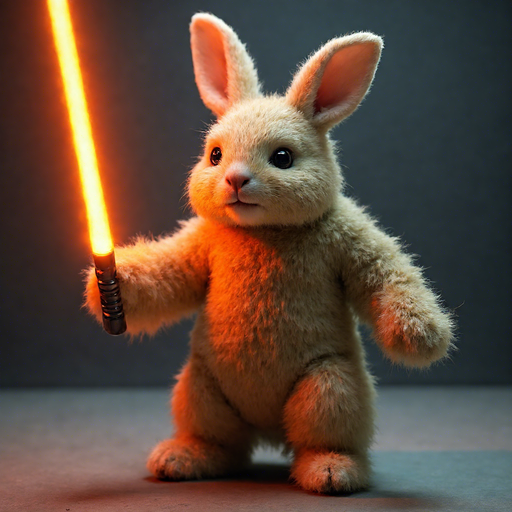} &
\includegraphics[width=\qualimgsize]{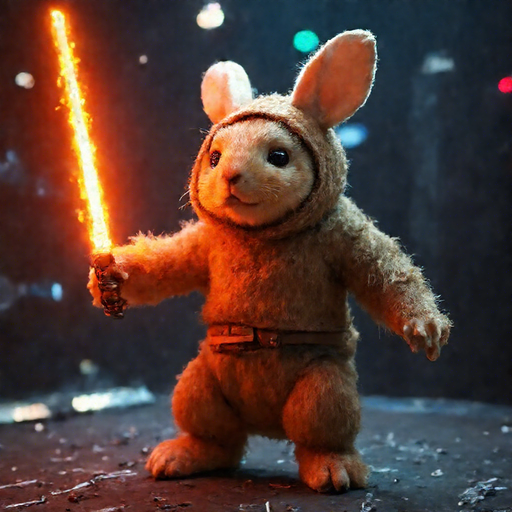} &
\includegraphics[width=\qualimgsize]{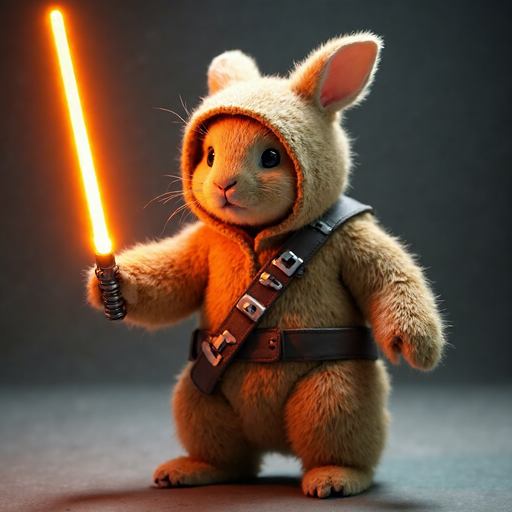} \\[-0.2em]

\promptline{a desktop holographic display showing a cute rabbit in a chewbacca onesie as a jedi with carrot lightsaber in the style of dieter rams}\\[0.5em]

ImageReward: 0.997 &
&
ImageReward: 0.905 &
ImageReward: 0.894 &
ImageReward: 0.755 &
ImageReward: 0.915 &
ImageReward: 1.060 \\[0.2em]

\includegraphics[width=\qualimgsize]{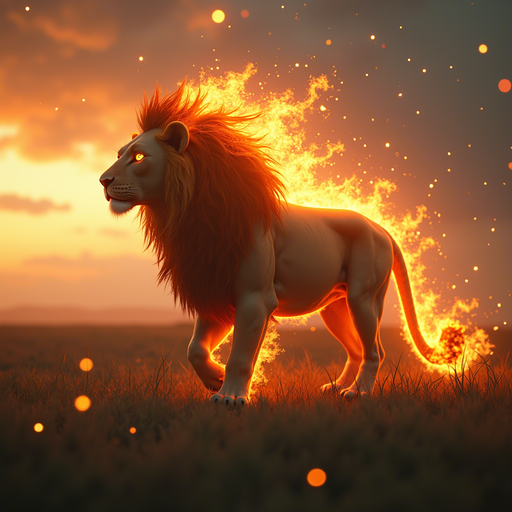} &
\imgvsep &
\includegraphics[width=\qualimgsize]{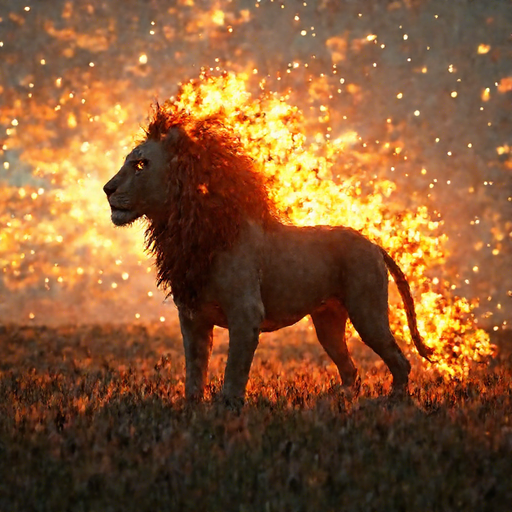} &
\includegraphics[width=\qualimgsize]{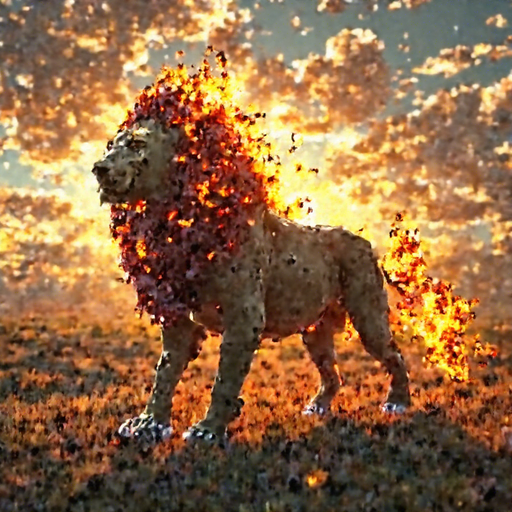} &
\includegraphics[width=\qualimgsize]{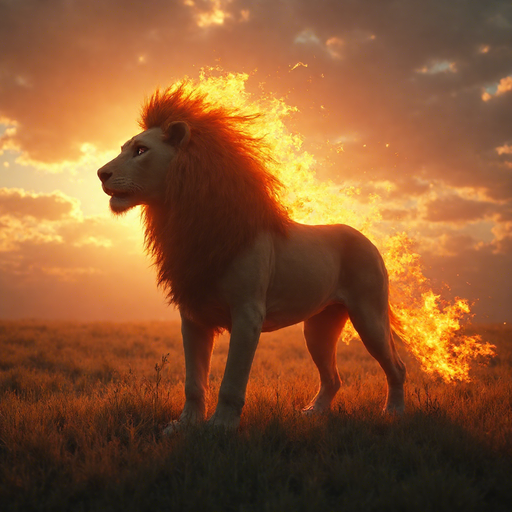} &
\includegraphics[width=\qualimgsize]{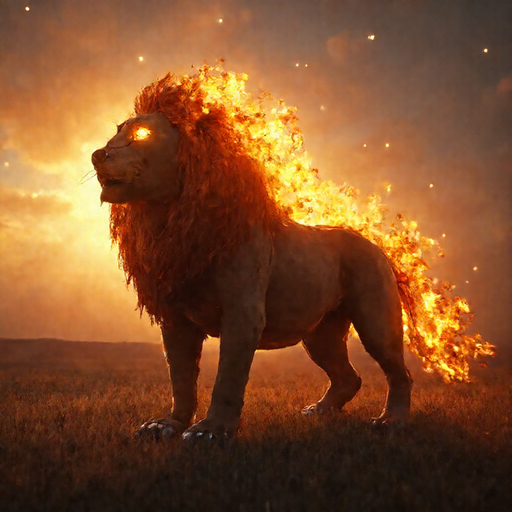} &
\includegraphics[width=\qualimgsize]{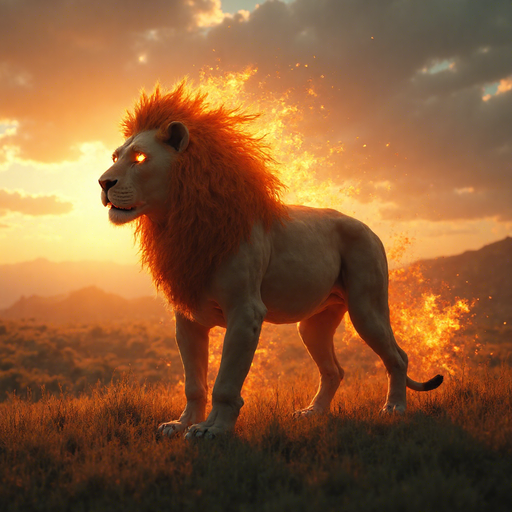} \\[-0.2em]

\promptline{A lion God with eyes and a mane of living fire looking over a vast grasslands imagination, inception style, detailed realistic, Unreal Engine, Cinematic, Color Grading, fullbody Photography...}\\[1.2mm]

\multicolumn{7}{c}{\textbf{(b) FLUX.1-dev}} \\[-0.2mm]

\end{tabular}
\end{adjustbox}

\vspace{-1mm}
\caption{
Qualitative comparison on SANA-1.5-1.6B and FLUX.1-dev.
All methods are evaluated under the W3A3 setting. The Image Reward score of each generated image is shown in the figure.
}
\label{fig:main_visual_results}
\vspace{-5mm}
\end{figure*}

\vspace{-2mm}
\subsection{Comparison with State-of-the-art Methods}
\vspace{-2mm}
We compare our method with representative PTQ methods, including RTN~\cite{nagel2021white}, SmoothQuant~\cite{xiao2023smoothquant}, OmniQuant~\cite{OmniQuant}, SVDQuant~\cite{li2024svdquant} and ConvRot~\cite{huang2025convrot}.
All methods are evaluated under 3-bit symmetric per-group quantization on three diffusion models: Z-Image-Turbo~\cite{cai2025z}, FLUX.1-dev~\cite{flux2024}, and SANA-1.5-1.6B~\cite{xie2025sana}.
We report two variants of our method, \textbf{PermuQuant} and \textbf{PermuQuant-H}, where the latter applies reordering after the Hadamard transformation~\cite{ashkboos2024quarot}.

\textbf{Quantitative results.}
Table~\ref{tab:main_results} reports the quantitative comparison on MJHQ30K and GenEval.
Across the three evaluated diffusion models, direct 3-bit RTN causes severe degradation, especially on FID and KID.
Existing PTQ methods alleviate this issue, but their gains vary.
Our method consistently improves low-bit performance.
On Z-Image-Turbo, \textbf{PermuQuant-H} reduces FID from 86.5 to 80.4 and KID from 14.1 to 12.3 compared with ConvRot, while improving the GenEval overall score from 0.81 to 0.83.
For FLUX.1-dev, \textbf{PermuQuant-H} achieves the best FID and KID among quantized methods and matches the best GenEval overall score.
On SANA-1.5-1.6B, \textbf{PermuQuant} already improves the GenEval overall score from 0.68 to 0.71 over ConvRot.
With Hadamard transformation, \textbf{PermuQuant-H} further reduces FID from 98.1 to 68.0 and KID from 24.8 to 5.50, while improving ImageReward from 0.440 to 0.919.
These results show that our reordering strategy is effective across different diffusion backbones and can be further strengthened by Hadamard transformation.

\textbf{Qualitative results.}
We show qualitative comparisons in Fig.~\ref{fig:main_visual_results}.
Under W3A3 quantization, existing PTQ methods can produce noticeable generation artifacts, such as noisy backgrounds, unstable textures, and degraded color consistency.
They may also weaken the semantic alignment with the prompt, especially for fine-grained details.
In contrast, our method generates images that are visually closer to the BF16 outputs.
It better preserves object structures, colors, and styles across different prompts and diffusion backbones.
More visual results are provided in Appendix~\ref{app:visual_quality_results}.

\vspace{-2mm}
\subsection{Ablation Study}
\vspace{-2mm}
We conduct ablation studies on Z-Image-Turbo to analyze the main design choices of our method.
The first three studies are performed with Hadamard transformation, while the group-size study is performed without Hadamard transformation.
The results are summarized in Tab.~\ref{tab:ablation}.

\textbf{Reordering criterion.}
We first compare different criteria for channel reordering.
Sorting by the second moment achieves the best FID and ImageReward.
Compared with variance-based sorting, it improves ImageReward from 0.629 to 0.715.
Random, absmax, and round-robin ordering are less effective.
This supports our choice of second moment as the reordering criterion.

\begin{table}[t]
\centering
\caption{Ablation studies. We report FID, ImageReward (IR), and GenEval overall score. B and R denote the baseline and reordered variants, respectively.}
\label{tab:ablation}
\renewcommand{\arraystretch}{1.12}

\begin{minipage}[t]{0.64\linewidth}
\centering
\scalebox{0.75}{
\begin{tabular}{cccccc}
\toprule
\rowcolor{gray!18}
Metric & Random & Absmax & Round-Robin & Variance & Second Moment \\
\midrule
FID $\downarrow$     & 89.0 & 85.3 & 83.5 & 83.3 & 80.4 \\
IR $\uparrow$        & 0.608 & 0.638 & 0.630 & 0.629 & 0.715 \\
GenEval $\uparrow$   & 0.83 & 0.84 & 0.77 & 0.82 & 0.83 \\
\bottomrule
\end{tabular}
}\\[0.5mm]
{\footnotesize (a) Reordering criterion.}
\end{minipage}
\hfill
\begin{minipage}[t]{0.35\linewidth}
\centering
\scalebox{0.75}{
\begin{tabular}{cccc}
\toprule
\rowcolor{gray!18}
Metric & Weight & Activation & Joint \\
\midrule
FID $\downarrow$     & 87.7 & 83.2 & 80.4 \\
IR $\uparrow$        & 0.621 & 0.647 & 0.715 \\
GenEval $\uparrow$   & 0.82 & 0.84 & 0.83 \\
\bottomrule
\end{tabular}
}\\[0.5mm]
{\footnotesize (b) Source of statistics.}
\end{minipage}

\vspace{1.5mm}

\begin{minipage}[t]{0.36\linewidth}
\centering
\scalebox{0.72}{
\begin{tabular}{ccccc}
\toprule
\rowcolor{gray!18}
Metric &~ No $\tau$ ~&~ $\tau$=0 ~&~~ $\tau$=2.5 ~~&~ $\tau$=5 ~\\
\midrule
FID $\downarrow$     & 85.4 & 80.4 & 80.3 & 82.8 \\
IR $\uparrow$        & 0.624 & 0.715 & 0.679 & 0.635 \\
GenEval $\uparrow$   & 0.83 & 0.83 & 0.83 & 0.81 \\
\bottomrule
\end{tabular}
}
\\[0.5mm]
{\footnotesize (c) Threshold choices.}
\end{minipage}
\hfill
\begin{minipage}[t]{0.60\linewidth}
\centering
\scalebox{0.72}{
\begin{tabular}{c|cc|cc|cc}
\toprule
\rowcolor{gray!18}
Metric 
& $g$=32 B & $g$=32 R
& $g$=128 B & $g$=128 R
& $g$=256 B & $g$=256 R \\
\midrule
FID $\downarrow$     & 251 & 135 & 77.8 & 74.3 & 93.9 & 85.8 \\
IR $\uparrow$        & -1.61 & -0.163 & 0.814 & 0.846 & 0.533 & 0.563 \\
GenEval $\uparrow$   & 0.32 & 0.74 & 0.81 & 0.83 & 0.82 & 0.84 \\
\bottomrule
\end{tabular}
}
\\[0.5mm]
{\footnotesize (d) Group size.}
\end{minipage}

\end{table}

\begin{figure}[t]
    \centering
    \vspace{-2mm}
    \includegraphics[width=1.0\linewidth]{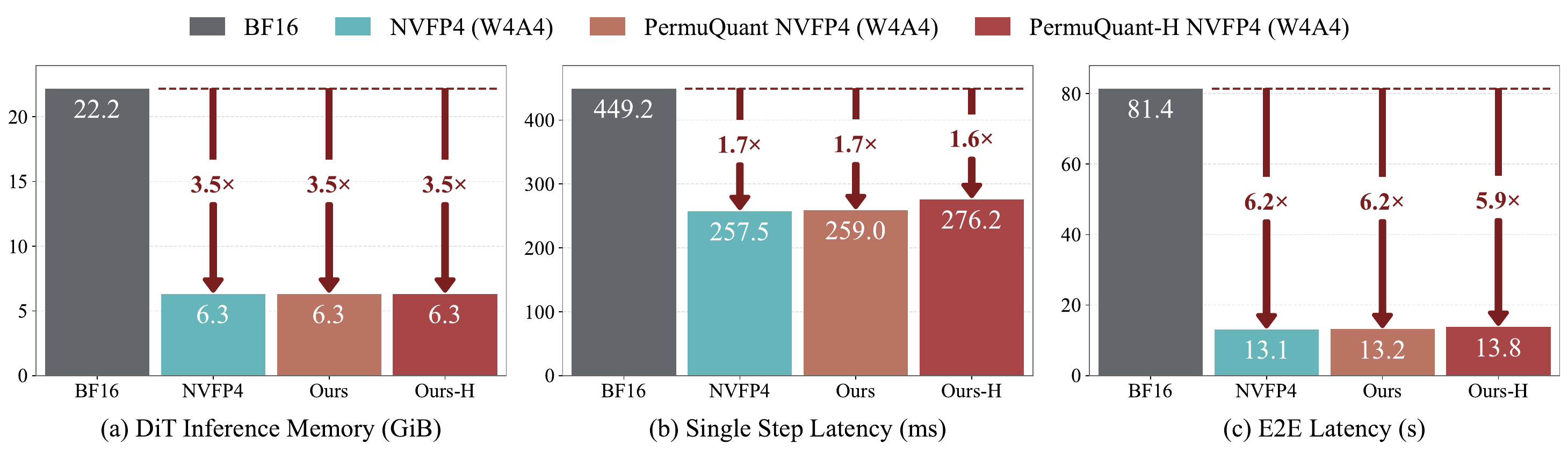}
    \caption{
        Efficiency comparison on FLUX.1-dev using an RTX 5090 Desktop 32GB GPU.
        Reordering introduces almost no extra overhead compared with NVFP4. By eliminating CPU offloading, our NVFP4 variants achieve up to a 6.3$\times$ E2E latency speedup over the BF16 baseline.
    }
    \label{fig:speedup}
    \vspace{-4mm}
\end{figure} 

\textbf{Source of statistics.}
We next study whether the reordering statistics should be computed from weights, activations, or both according to our joint sorting strategy.
Activation statistics outperform weight statistics, and joint second-moment reordering further improves FID and ImageReward.
This shows the benefit of accounting for both activation and weight quantization difficulty.

\textbf{Threshold choices.}
We also evaluate different thresholds for accepting a reordered permutation.
Without the acceptance threshold, reordering can hurt performance, leading to worse FID and ImageReward.
Using $\tau=0$ gives the best ImageReward and achieves strong FID and GenEval scores.
A moderate threshold, $\tau=2.5$, gives a similar FID and the same GenEval score.
However, a larger threshold degrades the results.
We therefore use $\tau=0$ as the default setting.

\textbf{Group size.}
Finally, we evaluate the effect of reordering under different group sizes.
For $g=32$, we use the 3-bit setting.
For $g=128$ and $g=256$, we use the 4-bit setting.
Reordering consistently improves the baseline under all group sizes.
The gain is especially large for $g=32$, where FID decreases from 251 to 135 and GenEval improves from 0.32 to 0.74.
These results show that our method is effective across different quantization granularities.

\vspace{-2mm}
\subsection{Efficiency Comparison}
\vspace{-2mm}
In Fig.~\ref{fig:speedup}, we report the measured DiT inference memory, single-step latency, and end-to-end (E2E) latency for FLUX.1-dev on an RTX 5090 Desktop GPU with 32GB memory.
Compared with the BF16 baseline, NVFP4 reduces the DiT inference memory from 22.2 GiB to 6.3 GiB, corresponding to a 3.5$\times$ reduction.
Our method retains the same memory footprint and introduces almost no extra overhead over NVFP4.
It also achieves nearly identical single-step latency, with 257.5 ms for NVFP4, 259.0 ms for PermuQuant, and 276.2 ms for PermuQuant-H.
For E2E latency, our NVFP4 variants reduce the latency from 81.4 s to less than 14 s, achieving up to a 6.3$\times$ speedup over the BF16 baseline. Here, E2E latency includes text encoding, DiT denoising, and VAE decoding.
Notably, this large E2E gain is partly because quantization eliminates the CPU offloading required by BF16 inference, allowing the model to remain entirely in GPU memory.

\vspace{-2mm}
\section{Conclusion}
\vspace{-2mm}
\label{sec:conclusion}

In this paper, we propose PermuQuant, a simple and effective low-bit quantization framework for diffusion models based on channel reordering. Our key observation is that, under per-group quantization, the channel order determines the induced group partitions and therefore has a substantial impact on quantization error. Based on this observation, PermuQuant groups channels with similar activation and weight statistics together to reduce within-group range mismatch. We then use a calibration-based rule to apply reordering only when it reliably lowers quantization error. Extensive experiments across multiple large diffusion models show that PermuQuant consistently preserves generation quality even under extremely low-bit quantization.
On FLUX.1-dev with an RTX 5090, PermuQuant achieves up to a \textbf{1.7}$\times$ single step speedup and reduces the DiT memory footprint by \textbf{3.5}$\times$ under W4A4 NVFP4 quantization. These results highlight channel reordering as a simple and practical direction for low-bit diffusion model compression.

\bibliographystyle{plainnat}
\bibliography{references} 

%%%%%%%%%%%%%%%%%%%%%%%%%%%%%%%%%%%%%%%%%%%%%%%%%%%%%%%%%%%%
\clearpage
\appendix
\section{Table of Contents}
\label{app:contents}
In the supplementary material, we provide complete proofs, implementation details, analysis, and results, including:
\begin{itemize}
    \item Sec.~\ref{app:proofs}: Proofs of the expected quantization error bound, the optimality of second-moment sorting, and the approximation guarantee.
    \item Sec.~\ref{app:detail_deploy}: Additional details for runtime-free deployment.
    \item Sec.~\ref{app:reordering_hadamard_basis}: Discussion of how to combine reordering with the Hadamard transform.
    \item Sec.~\ref{app:implementation_detail}: Implementation details.
    \item Sec.~\ref{app:limitation}: Limitations of PermuQuant.
    \item Sec.~\ref{app:broader_impact}: Broader impact.
    \item Sec.~\ref{app:additional_results}: Additional results, including the ablation of Hadamard transformation and additional visual quality comparisons.
\end{itemize}

\section{Proofs}
\label{app:proofs}

\subsection{Proof of Eq.~\eqref{eq:exp_error_pre}}
\label{app:proof_eq6}

We prove that
\begin{equation}
    \mathbb{E}_{\mathbf{x}}\|\mathbf{e}\|_2^2
    \le
    \frac{g}{4Q^2}
    \sum_k
    \mathbb{E}_{\mathbf{x}}\left[
    \max_{i\in G_k}|x_i|^2
    \right].
\end{equation}

\noindent\textit{Proof.}
Let $e_i=x_i-\hat{x}_i$ denote the quantization error of the $i$-th element.
For each group $G_k$, the quantization scale is
\begin{equation}
    s_k = \frac{\max_{i\in G_k}|x_i|}{Q}.
\end{equation}
Under round-to-nearest quantization, $|e_i|\le s_k/2$ for $i\in G_k$.
Thus,
\begin{equation}
\begin{aligned}
    \|\mathbf{e}\|_2^2
    &= \sum_k \sum_{i\in G_k} e_i^2 \\
    &\le \sum_k \sum_{i\in G_k} \frac{s_k^2}{4} \\
    &= \sum_k \frac{g}{4}
    \left(\frac{\max_{i\in G_k}|x_i|}{Q}\right)^2 \\
    &= \frac{g}{4Q^2}
    \sum_k \max_{i\in G_k}|x_i|^2 .
\end{aligned}
\end{equation}
Taking expectation over $\mathbf{x}$ gives
\begin{equation}
    \mathbb{E}_{\mathbf{x}}\|\mathbf{e}\|_2^2
    \le
    \frac{g}{4Q^2}
    \sum_k
    \mathbb{E}_{\mathbf{x}}\left[
    \max_{i\in G_k}|x_i|^2
    \right],
\end{equation}
which proves Eq.~\eqref{eq:exp_error_pre}.

\subsection{Proof of Prop.~\ref{prop:second_moment_sorting}}
\label{app:proof_optimal_bound}

We prove that sorting channels by $\mu_i^2$ and then partitioning them into contiguous groups minimizes
\begin{equation}
    \sum_k \max_{i\in G_k}\mu_i^2 .
\end{equation}

\noindent\textit{Proof.}
Let $a_i=\mu_i^2$ and assume, without loss of generality, that
\begin{equation}
    a_1 \ge a_2 \ge \cdots \ge a_d .
\end{equation}
Let the group size be $g$ and $K=d/g$.
Under the sorted partition, the group maxima are
\begin{equation}
    a_1,\, a_{g+1},\, a_{2g+1},\,\ldots,\,a_{(K-1)g+1}.
\end{equation}
Thus, its objective value is
\begin{equation}
    S_{\mathrm{sort}}
    =
    \sum_{k=1}^{K} a_{(k-1)g+1}.
\end{equation}

For any partition into $K$ groups of size $g$, consider the first $(k-1)g+1$ largest elements
\begin{equation}
    a_1,\ldots,a_{(k-1)g+1}.
\end{equation}
Since each group contains at most $g$ elements, these elements must occupy at least $k$ different groups.
Therefore, the $k$-th largest group maximum of any partition is at least
\begin{equation}
    a_{(k-1)g+1}.
\end{equation}
Let $m_1\ge m_2\ge\cdots\ge m_K$ be the group maxima of any partition sorted in descending order.
Then
\begin{equation}
    m_k \ge a_{(k-1)g+1}, \quad k=1,\ldots,K.
\end{equation}
It follows that
\begin{equation}
    \sum_{k=1}^{K} m_k
    \ge
    \sum_{k=1}^{K} a_{(k-1)g+1}
    =
    S_{\mathrm{sort}}.
\end{equation}
Thus, the sorted partition minimizes $\sum_k \max_{i\in G_k}\mu_i^2$.

\subsection{Proof of Theorem~\ref{thm:second_moment_sorting}}
\label{app:proof_permuquant}

\noindent\textbf{Theorem~\ref{thm:second_moment_sorting}
(Approximation guarantee).}
\textit{Let $\mathcal{E}(P):=\mathbb{E}_{x}\mathbb{E}_{\epsilon\mid x}\|\epsilon\|_2^2$ denote the expected quantization error under the uniform-noise approximation~\cite{marco2005validity}. $\epsilon$ is the additive quantization noise and $\epsilon_i\mid x \sim \mathrm{Unif}(-s_k/2,s_k/2)$ for $i\in G_k$, with $s_k=\max_{j\in G_k}|x_j|/Q$. Let $P_{\mathrm{sort}}$ be the partition obtained by sorting channels in descending order of $\mu_i^2$, where $\mathbb{E}[x_i^2] = \mu_i^2$. Assume that Assumption~\ref{assump:extremal_control} holds. Then}
\begin{equation}
    \mathcal{E}(P_{\mathrm{sort}})
    \le
    C \log(2g)\cdot \min_P \mathcal{E}(P).
\end{equation}

\noindent\textit{Proof.}
Let $P=\{G_1,\ldots,G_K\}$ be a partition of the $d$ channels into
$K=d/g$ groups, where each group has size $g$.
For each group $G_k$, define
\begin{equation}
    M_k := \max_{i\in G_k} |x_i|.
\end{equation}
Under online dynamic per-group symmetric quantization, the quantization
step size of group $G_k$ is
\begin{equation}
    s_k = \frac{M_k}{Q},
\end{equation}
where $Q$ is the largest representable integer magnitude determined by
the bit-width.

Under the uniform-noise approximation, the quantization error is modeled
as additive noise $\epsilon$. Conditioned on the input activation $x$, for
each $i\in G_k$, we have
\begin{equation}
    \epsilon_i \mid x \sim \mathrm{Unif}\left(-\frac{s_k}{2}, \frac{s_k}{2}\right).
\end{equation}
Therefore,
\begin{equation}
    \mathbb{E}_{\epsilon\mid x}\left[\epsilon_i^2\right]
    =
    \frac{s_k^2}{12},
    \qquad i\in G_k .
\end{equation}
It follows that the conditional expected noise energy of group $G_k$ is
\begin{equation}
    \mathbb{E}_{\epsilon\mid x}
    \left[
        \sum_{i\in G_k}\epsilon_i^2
    \right]
    =
    \sum_{i\in G_k}
    \mathbb{E}_{\epsilon\mid x}
    \left[\epsilon_i^2\right]
    =
    g\cdot \frac{s_k^2}{12}
    =
    \frac{g}{12Q^2}M_k^2 .
\end{equation}
Taking expectation over the activation distribution and summing over all
groups gives
\begin{equation}
    \mathcal{E}(P)
    =
    \mathbb{E}_{x}\mathbb{E}_{\epsilon\mid x}\|\epsilon\|_2^2
    =
    \frac{g}{12Q^2}
    \sum_{k=1}^{K}
    \mathbb{E}_{x}\left[M_k^2\right].
    \label{eq:app_error_group_max}
\end{equation}

Next, by Assumption~\ref{assump:extremal_control}, for every group
$G_k$,
\begin{equation}
    \mathbb{E}_{x}\left[M_k^2\right]
    =
    \mathbb{E}_{x}\left[\max_{i\in G_k}|x_i|^2\right]
    \le
    C\log(2g)\cdot \max_{i\in G_k}\mu_i^2 ,
\end{equation}
where $\mu_i^2=\mathbb{E}_{x}[x_i^2]$ denotes the second moment of
channel $i$.
Substituting this bound into Eq.~\eqref{eq:app_error_group_max}, we get
\begin{equation}
    \mathcal{E}(P)
    \le
    \frac{Cg\log(2g)}{12Q^2}
    \sum_{k=1}^{K}
    \max_{i\in G_k}\mu_i^2 .
    \label{eq:app_upper_bound}
\end{equation}

We also need a matching lower bound. For any group $G_k$, let
\begin{equation}
    i_k^* \in \arg\max_{i\in G_k}\mu_i^2 .
\end{equation}
Since
\begin{equation}
    M_k^2
    =
    \max_{i\in G_k}|x_i|^2
    \ge
    |x_{i_k^*}|^2 ,
\end{equation}
taking expectation over $x$ gives
\begin{equation}
    \mathbb{E}_{x}\left[M_k^2\right]
    \ge
    \mathbb{E}_{x}\left[|x_{i_k^*}|^2\right]
    =
    \mu_{i_k^*}^2
    =
    \max_{i\in G_k}\mu_i^2 .
\end{equation}
Substituting this into Eq.~\eqref{eq:app_error_group_max}, we obtain
\begin{equation}
    \mathcal{E}(P)
    \ge
    \frac{g}{12Q^2}
    \sum_{k=1}^{K}
    \max_{i\in G_k}\mu_i^2 .
    \label{eq:app_lower_bound}
\end{equation}

Define the proxy objective
\begin{equation}
    \Phi(P)
    :=
    \sum_{k=1}^{K}
    \max_{i\in G_k}\mu_i^2 .
\end{equation}
Equations~\eqref{eq:app_upper_bound} and~\eqref{eq:app_lower_bound}
imply the sandwich bound
\begin{equation}
    \frac{g}{12Q^2}\Phi(P)
    \le
    \mathcal{E}(P)
    \le
    \frac{Cg\log(2g)}{12Q^2}\Phi(P).
    \label{eq:app_sandwich}
\end{equation}

By the optimality of second-moment sorting shown in
Appendix~\ref{app:proof_optimal_bound}, $P_{\mathrm{sort}}$ minimizes
the proxy objective $\Phi(P)$. Therefore, for any valid partition $P$,
\begin{equation}
    \Phi(P_{\mathrm{sort}}) \le \Phi(P).
\end{equation}

Let
\begin{equation}
    P^* \in \arg\min_P \mathcal{E}(P)
\end{equation}
be the optimal partition with respect to $\mathcal{E}(P)$.
Since $P_{\mathrm{sort}}$ minimizes $\Phi(P)$, we have
\begin{equation}
    \Phi(P_{\mathrm{sort}})
    \le
    \Phi(P^*) .
\end{equation}
Using the upper bound in Eq.~\eqref{eq:app_sandwich}, we get
\begin{equation}
    \mathcal{E}(P_{\mathrm{sort}})
    \le
    \frac{Cg\log(2g)}{12Q^2}
    \Phi(P_{\mathrm{sort}})
    \le
    \frac{Cg\log(2g)}{12Q^2}
    \Phi(P^*) .
\end{equation}
Using the lower bound in Eq.~\eqref{eq:app_sandwich} on $P^*$ gives
\begin{equation}
    \Phi(P^*)
    \le
    \frac{12Q^2}{g}\mathcal{E}(P^*) .
\end{equation}
Combining the two inequalities yields
\begin{equation}
    \mathcal{E}(P_{\mathrm{sort}})
    \le
    C\log(2g)\cdot \mathcal{E}(P^*)
    =
    C\log(2g)\cdot \min_P \mathcal{E}(P).
\end{equation}
This proves that second-moment based sorting is a
$C\log(2g)$-approximation to the optimal partition with respect to the
expected quantization error under the uniform-noise approximation.

\section{Additional Detail for Run-time Free Deployment.}
\label{app:detail_deploy}

As illustrated in Fig.~\ref{fig:permuquant}(b), the activation-side permutation before each linear layer can be absorbed either into its preceding linear layer or into its preceding normalization layer.
The weight-side permutation can also be pre-applied offline.
Therefore, PermuQuant does not require an explicit permutation operator during inference.
Here, we provide the technical details for the normalization case.
In particular, we show that channel reordering can be fused into RMSNorm or LayerNorm with negligible additional runtime cost.
This is possible because both RMSNorm and LayerNorm are permutation-equivariant along the normalized channel dimension.
Let $P \in \{0,1\}^{d \times d}$ denote a channel permutation matrix acting on a token vector $x \in \mathbb{R}^d$, where $d$ is the channel dimension.

\textbf{RMSNorm.}
RMSNorm computes
\begin{equation}
\mathrm{RMSNorm}(x; \gamma)
=
\frac{x}{\sqrt{\frac{1}{d}\sum_{i=1}^{d} x_i^2 + \varepsilon}}
\odot \gamma ,
\end{equation}
where $\gamma \in \mathbb{R}^d$ is the channel-wise scale.
Since a permutation does not change the sum of squares,
\begin{equation}
\sum_{i=1}^{d} (Px)_i^2
=
\sum_{i=1}^{d} x_i^2 ,
\end{equation}
the normalization factor is invariant to channel reordering.
Therefore,
\begin{equation}
P\,\mathrm{RMSNorm}(x; \gamma)
=
\mathrm{RMSNorm}(Px; P\gamma).
\end{equation}

\textbf{LayerNorm.}
LayerNorm computes
\begin{equation}
\mathrm{LN}(x)
=
\frac{x-\mu(x)}{\sqrt{\sigma^2(x)+\varepsilon}},
\qquad
\mu(x)
=
\frac{1}{d}\sum_{i=1}^{d} x_i,
\qquad
\sigma^2(x)
=
\frac{1}{d}\sum_{i=1}^{d}(x_i-\mu(x))^2 .
\end{equation}
Both the mean and variance are invariant to channel permutation:
\begin{equation}
\mu(Px)=\mu(x),
\qquad
\sigma^2(Px)=\sigma^2(x).
\end{equation}
Thus,
\begin{equation}
P\,\mathrm{LN}(x)
=
\mathrm{LN}(Px).
\end{equation}
If LayerNorm is followed by channel-wise modulation,
\begin{equation}
y = (1+s)\odot \mathrm{LN}(x) + b ,
\end{equation}
then the same permutation can be applied to the modulation vectors:
\begin{equation}
P y
=
(1+Ps)\odot \mathrm{LN}(Px) + Pb .
\end{equation}
Therefore, for LayerNorm with adaptive scale and shift, channel reordering can be moved before normalization as long as the modulation vectors are reordered accordingly.

\textbf{Why the fusion is nearly zero-cost.}
A naive implementation first computes normalization and then launches a separate gather kernel:
\begin{equation}
x
\xrightarrow{\mathrm{Norm}}
y
\xrightarrow{\mathrm{Reorder}}
y_P .
\end{equation}
This introduces an additional activation read, an additional activation write, and an extra kernel launch.
In contrast, our fused implementation performs
\begin{equation}
x
\xrightarrow{\text{indexed load}}
x_P
\xrightarrow{\mathrm{Norm\;in\;registers}}
y_P
\end{equation}
within a single normalization kernel.

The key observation is that normalization already needs to read the full input row.
We therefore replace the original sequential load with an indexed load according to the permutation, while keeping the remaining computation unchanged.
Specifically, the fused kernel:
\begin{enumerate}
    \item loads one activation row as $x[\pi(1)],\ldots,x[\pi(K)]$ from global memory;
    \item computes the RMS statistic or the mean and variance on the loaded values;
    \item applies the corresponding channel-wise scale or modulation;
    \item writes the reordered normalized output contiguously.
\end{enumerate}
In this way, the reordering is absorbed into the mandatory input-read stage of normalization.
The fused kernel avoids a standalone reorder pass over the activation tensor.
The only additional work is reading the permutation indices and generating indexed memory addresses.
This is much cheaper than materializing an intermediate tensor and performing an extra global-memory gather and writeback.
Thus, the activation-side permutation is nearly zero-cost at runtime.
For adaptive LayerNorm, which is commonly used in DiT Models~\cite{peebles2023scalable}, the same strategy applies to the dynamic modulation parameters. 
During the kernel load stage, the per-channel scale and shift are gathered on the fly using the same permutation. 
This avoids materializing reordered modulation tensors through a separate pass.

\textbf{Implementation detail.}
In our Triton implementation, each program processes one row of the normalized tensor.
When reordering is enabled, the kernel first loads the permutation indices $\pi(j)$.
It then directly issues indexed global-memory loads, rather than materializing a reordered tensor:
\begin{equation}
x_j \leftarrow x[\pi(j)] .
\end{equation}
The subsequent normalization and channel-wise affine or modulation operations are computed on these loaded values.
The final output is stored contiguously.
As a result, channel reordering does not appear as a separate operator.
It is absorbed into the row load of RMSNorm or LayerNorm.

\section{Reordering in the Hadamard-Transformed Basis}
\label{app:reordering_hadamard_basis}

In PermuQuant-H, reordering is performed after the Hadamard transformation.
This is because per-group quantization is applied to the Hadamard-transformed representation.
Let $\mathbf{H}$ denote the Hadamard transform and $\mathbf{P}$ denote the permutation matrix.
If reordering is applied before the Hadamard transformation, the representation becomes $\mathbf{x}\mathbf{P}\mathbf{H}$.
The Hadamard transform mixes the permuted channels again.
Thus, the pre-Hadamard order does not directly determine the contiguous groups that are quantized.
In contrast, applying reordering after the Hadamard transformation gives $\mathbf{x}\mathbf{H}\mathbf{P}$.
The permutation then directly determines which transformed channels share a per-group quantization scale.
Therefore, the reordering is aligned with the actual representation being quantized.

This design does not require explicitly storing the matrix $\mathbf{H}\mathbf{P}$.
In practice, the Hadamard transform is implemented by a fast Hadamard transform rather than by materializing $\mathbf{H}$.
The post-Hadamard permutation can be fused into the output layout of the Hadamard kernel.
Equivalently, the Hadamard computation is kept unchanged, and only the order in which transformed channels are written or consumed is modified.
The additional cost is therefore limited to lightweight index mapping, rather than an extra matrix multiplication.
When implemented as part of the Hadamard, reordering avoids an extra global-memory pass and preserves the runtime-free property up to negligible indexing overhead.

\section{Implementation Details}
\label{app:implementation_detail}
We conduct experiments under the 3-bit setting.
We adopt per-group symmetric quantization for both activations and weights.
For FLUX.1-dev, we use a group size of 64.
For Z-Image-Turbo and SANA-1.5-1.6B, we use a group size of 32.
For joint second-moment reordering, we search $\alpha$ from $\{0,0.2,0.4,0.6,0.8,1.0\}$ and select the value with the smallest calibration quantization error.
The acceptance threshold $\tau$ is set to 0 by default.
For ConvRot, the group size of grouped rotations is set to 64. For SmoothQuant, we compute the smoothing scales following its original formulation.
For each model, the smoothing strength $\alpha$ is selected by grid search on the calibration set.
For SVDQuant, we set the rank of the low-rank branch to 32, following the original SVDQuant setting.

\section{Limitations}
\label{app:limitation}
PermuQuant still has several limitations.
First, the current reordering strategy is heuristic.
Although our analysis provides an approximation guarantee under reasonable assumptions, the resulting permutation is not guaranteed to be globally optimal.
Future work may formulate channel reordering as an explicit optimization problem, enabling better permutation search or even optimal permutation design under practical constraints.
Second, PermuQuant relies on calibration data to estimate channel statistics and decide whether to apply reordering.
As in other PTQ methods, the quality of these estimates may depend on the representativeness of the calibration set.
Finally, our current method mainly focuses on linear layers in diffusion transformers.
Applying reordering to attention quantization remains unexplored.
Extending the reordering to these attention-specific quantization steps is an interesting and practical direction for future work.

\section{Broader Impact}
\label{app:broader_impact}
PermuQuant improves the efficiency of large diffusion models by reducing memory usage and inference latency under low-bit quantization.
This may lower the hardware barrier for deploying text-to-image generation models and reduce the energy cost of inference.
Therefore, our method can benefit researchers and practitioners with limited computational resources.

However, more efficient image generation may also lower the cost of misuse, such as generating misleading, harmful, or copyrighted visual content.
Our work does not directly address these risks.
Practical deployment should therefore be accompanied by appropriate safeguards, including content filtering, provenance tracking, watermarking, and responsible dataset governance.

\clearpage

\section{Additional Results}
\label{app:additional_results}

\subsection{Ablation of Hadamard Transform.}
\begin{table*}[t]
\centering
\caption{
Ablation study of reordering and Hadamard transformation on MJHQ30K and GenEval.
For MJHQ30K, we report FID, KID$\times 10^3$, and ImageReward (IR); for GenEval, we report the standard scores.
The best results among INT W3A3 variants are highlighted in \textbf{bold}.
}
\label{tab:ablation_hadamard}
\small
\setlength{\tabcolsep}{3pt}
\renewcommand{\arraystretch}{1.05}

\newcommand{\cmark}{\textcolor{green!60!black}{\ding{51}}}
\newcommand{\xmark}{\textcolor{red}{\ding{55}}}

\resizebox{\textwidth}{!}{
\begin{tabular}{c@{\hspace{8pt}}c@{\hspace{6pt}}cc@{\hspace{6pt}}cccccccccc}
\toprule
\multirow{2}{*}{Model}
& \multirow{2}{*}{Precision}
& \multirow{2}{*}{Reordering}
& \multirow{2}{*}{Hadamard}
& \multicolumn{3}{c}{MJHQ30K}
& \multicolumn{7}{c}{GenEval} \\
\cmidrule(lr){5-7} \cmidrule(lr){8-14}
& & & 
& FID $\downarrow$
& KID $\downarrow$
& IR $\uparrow$
& Single $\uparrow$
& Two $\uparrow$
& Count $\uparrow$
& Color $\uparrow$
& Pos. $\uparrow$
& Attr. $\uparrow$
& Overall $\uparrow$ \\
\midrule

\multirow{5}{1.2cm}{\centering \makecell{Z-Image\\-Turbo\\(8 Steps)}}
& BF16
& -- & --
& 67.6 & 5.7 & 1.001
& 1.00 & 0.97 & 0.71 & 1.00 & 0.63 & 0.83 & 0.86 \\
\cmidrule(l){2-14}

& \multirow{4}{*}{\makecell{INT\\W3A3}}
& \xmark & \xmark
& 250.8 & 183.0 & -1.609
& 0.69 & 0.21 & 0.21 & 0.47 & 0.21 & 0.10 & 0.32 \\

&
& \cmark & \xmark
& 134.9 & 49.2 & -0.163
& \textbf{0.96} & 0.87 & 0.61 & 0.93 & 0.55 & 0.50 & 0.74 \\

&
& \xmark & \cmark
& 84.2 & \textbf{12.0} & 0.608
& \textbf{0.96} & 0.95 & \textbf{0.71} & \textbf{1.00} & 0.61 & \textbf{0.78} & \textbf{0.83} \\

\rowcolor{lightblue}
\cellcolor{white}
& \cellcolor{white}
& \cmark & \cmark
& \textbf{80.4} & 12.3 & \textbf{0.715}
& \textbf{0.96} & \textbf{0.97} & 0.64 & \textbf{1.00} & \textbf{0.63} & 0.75 & \textbf{0.83} \\

\midrule

\multirow{5}{1.2cm}{\centering \makecell{FLUX.1\\-dev\\(50 Steps)}}
& BF16
& -- & --
& 60.6 & 2.8 & 0.917
& 1.00 & 0.89 & 0.61 & 0.93 & 0.68 & 0.53 & 0.77 \\
\cmidrule(l){2-14}

& \multirow{4}{*}{\makecell{INT\\W3A3}}
& \xmark & \xmark
& 131.0 & 52.4 & 0.386
& 0.81 & 0.45 & 0.29 & 0.67 & 0.32 & 0.10 & 0.44 \\

&
& \cmark & \xmark
& 74.9 & 9.1 & 0.726
& 0.96 & 0.84 & 0.50 & \textbf{0.97} & 0.58 & 0.40 & 0.71 \\

&
& \xmark & \cmark
& 68.4 & 5.0 & 0.730
& 0.96 & 0.87 & \textbf{0.64} & 0.90 & 0.68 & \textbf{0.63} & \textbf{0.78} \\

\rowcolor{lightblue}
\cellcolor{white}
& \cellcolor{white}
& \cmark & \cmark
& \textbf{64.7} & \textbf{4.6} & \textbf{0.759}
& \textbf{1.00} & \textbf{0.92} & 0.54 & 0.93 & \textbf{0.74} & 0.48 & 0.77 \\

\midrule

\multirow{5}{1.2cm}{\centering \makecell{SANA-1.5\\-1.6B\\(20 Steps)}}
& BF16
& -- & --
& 58.1 & 1.2 & 1.075
& 1.00 & 0.97 & 0.79 & 0.97 & 0.66 & 0.55 & 0.82 \\
\cmidrule(l){2-14}

& \multirow{4}{*}{\makecell{INT\\W3A3}}
& \xmark & \xmark
& 103.6 & 30.1 & 0.394
& 0.96 & 0.63 & 0.54 & 0.83 & 0.47 & 0.43 & 0.64 \\

&
& \cmark & \xmark
& 82.8 & 14.7 & 0.700
& 0.92 & 0.76 & \textbf{0.61} & 0.93 & 0.45 & 0.60 & 0.71 \\

&
& \xmark & \cmark
& 72.6 & 6.0 & 0.902
& \textbf{1.00} & 0.82 & 0.57 & 0.93 & 0.50 & \textbf{0.75} & 0.76 \\

\rowcolor{lightblue}
\cellcolor{white}
& \cellcolor{white}
& \cmark & \cmark
& \textbf{68.0} & \textbf{5.5} & \textbf{0.919}
& 0.96 & \textbf{0.89} & 0.57 & \textbf{0.97} & \textbf{0.66} & 0.73 & \textbf{0.80} \\

\bottomrule
\end{tabular}
}
\vspace{-3mm}
\end{table*}

Table~\ref{tab:ablation_hadamard} studies the effect of channel reordering and the Hadamard transform under INT W3A3 quantization. 
We compare four variants: the original W3A3 baseline, reordering only, Hadamard only, and their combination. 
Across all three models, both reordering and Hadamard transformation consistently improve over the direct W3A3 baseline. 
For example, on Z-Image-Turbo, reordering alone reduces FID from 250.8 to 134.9, while Hadamard alone further reduces it to 84.2. 
This shows that both operations effectively reduce low-bit quantization error.

More importantly, combining reordering with the Hadamard transform gives the best overall performance in most cases. 
On Z-Image-Turbo, the combined variant achieves the best FID, ImageReward, and GenEval overall score among INT W3A3 variants. 
Similar trends are observed on FLUX.1-dev and SANA-1.5-1.6B, where the combined variant obtains the lowest FID and KID, and also gives the strongest ImageReward. 
These results indicate that reordering is complementary to Hadamard-based rotation. 
While Hadamard transformation redistributes channel information globally, our reordering further improves the induced per-group partitions after rotation.

The two components also differ in efficiency cost. 
Reordering introduces almost no additional inference overhead, since the permutations can be folded into adjacent modules or pre-applied to the weights offline. 
By contrast, Hadamard transformation cannot always be fully fused in modern diffusion architectures. 
This is because these models commonly use AdaLN-based modulation~\cite{peebles2023scalable}, where part of the transformation interacts with input-dependent scaling and shifting operations. 
As a result, some Hadamard matrices must be applied during inference, which increases the end-to-end latency by about 8\% in our implementation. 
Thus, Hadamard transformation introduces an efficiency--quality trade-off in practice. 
Reordering alone is preferable when latency is critical, while the combined variant provides better quality when a small latency increase is acceptable.

\subsection{Visual Quality Results}
\label{app:visual_quality_results}

We provide additional visual comparisons in
Figs.~\ref{fig:sana_additional_vis}, \ref{fig:flux_additional_vis}, and
\ref{fig:zimage_additional_vis}.
These results complement the qualitative examples in the main paper and cover
three representative diffusion backbones: SANA-1.5-1.6B, FLUX.1-dev, and
Z-Image-Turbo.
For each model, the dashed line separates the BF16 reference from the quantized
methods.
All quantized methods are evaluated under the W3A3 setting.

\begin{figure*}[t]
\centering
\scriptsize
\setlength{\tabcolsep}{2pt}
\renewcommand{\arraystretch}{0.95}
\setlength{\qualimgsize}{0.16\textwidth}

\newcommand{\imgvsep}{
  \begin{tikzpicture}[baseline={(0,0)}]
    \draw[densely dashed, line width=0.4pt] (0,0) -- (0,\qualimgsize);
  \end{tikzpicture}
}

\newcommand{\promptline}[1]{%
\multicolumn{7}{c}{%
\parbox{0.92\textwidth}{%
\centering
\setlength{\baselineskip}{0.85\baselineskip}%
\textit{Prompt: #1}%
}%
}%
}

\begin{adjustbox}{width=\textwidth}
\begin{tabular}{c@{\hspace{0.5mm}}c c c c c c}

\textbf{BF16} & &
\textbf{OmniQuant} &
\textbf{SVDQuant} &
\textbf{ConvRot} &
\textbf{PermuQuant} &
\textbf{PermuQuant-H} \\

\includegraphics[width=\qualimgsize]{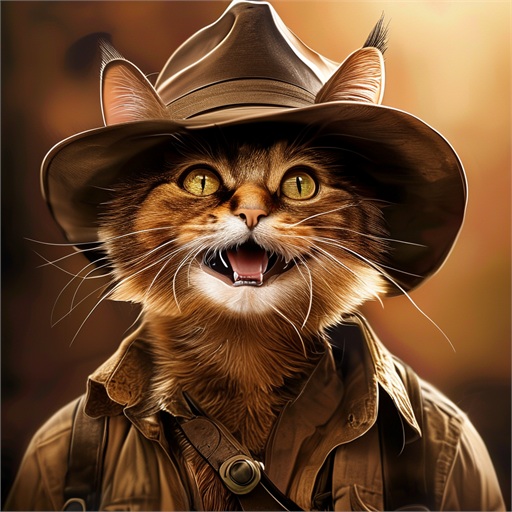} &
\imgvsep &
\includegraphics[width=\qualimgsize]{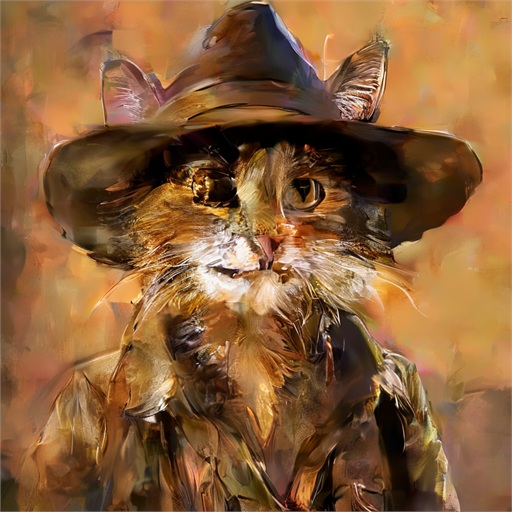} &
\includegraphics[width=\qualimgsize]{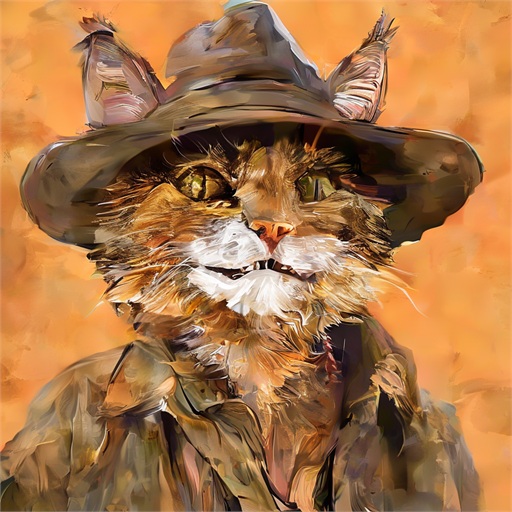} &
\includegraphics[width=\qualimgsize]{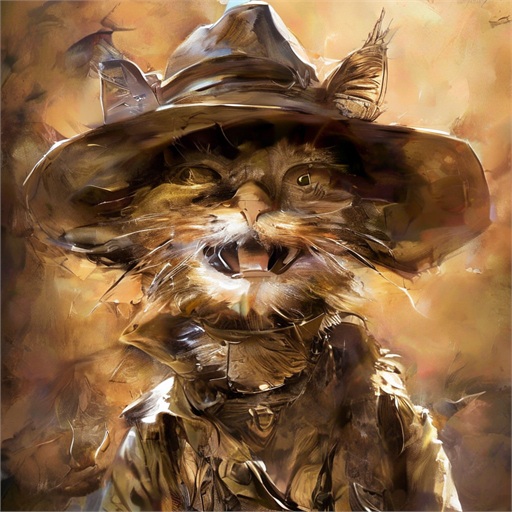} &
\includegraphics[width=\qualimgsize]{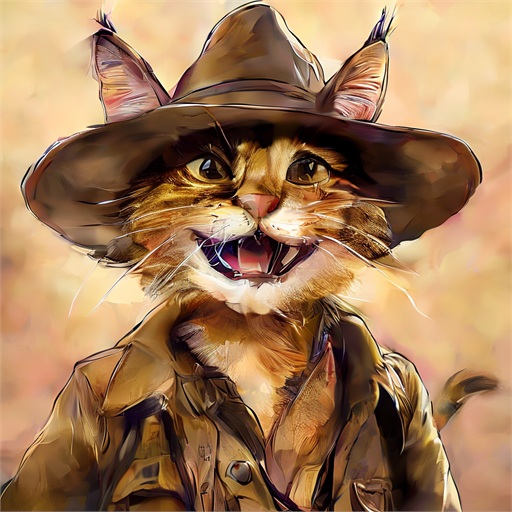} &
\includegraphics[width=\qualimgsize]{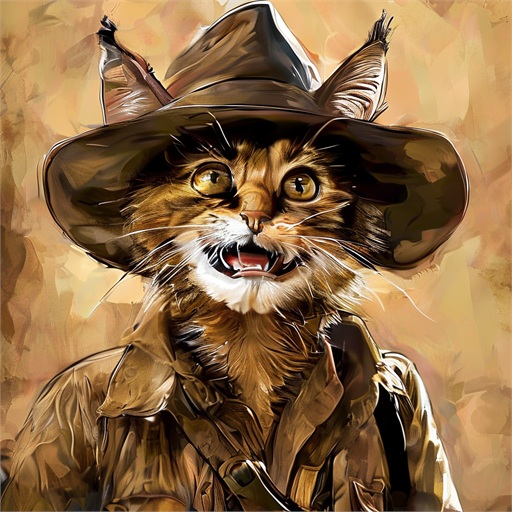} \\[-0.2em]

\promptline{indiana jones themed cat, wearing hat, happy and adventurous}\\[0.4em]

\includegraphics[width=\qualimgsize]{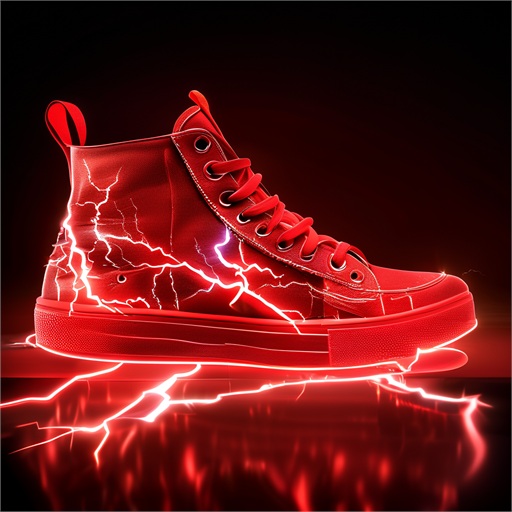} &
\imgvsep &
\includegraphics[width=\qualimgsize]{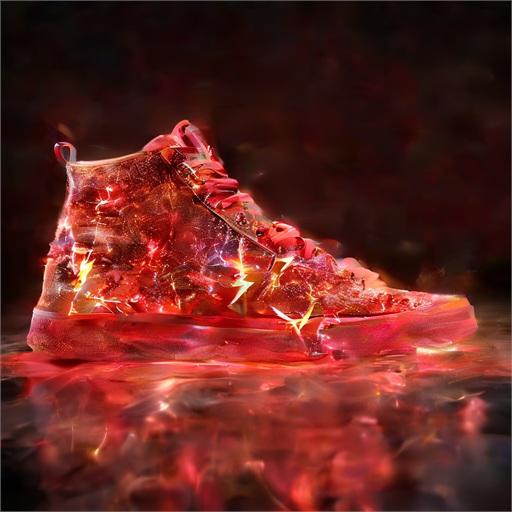} &
\includegraphics[width=\qualimgsize]{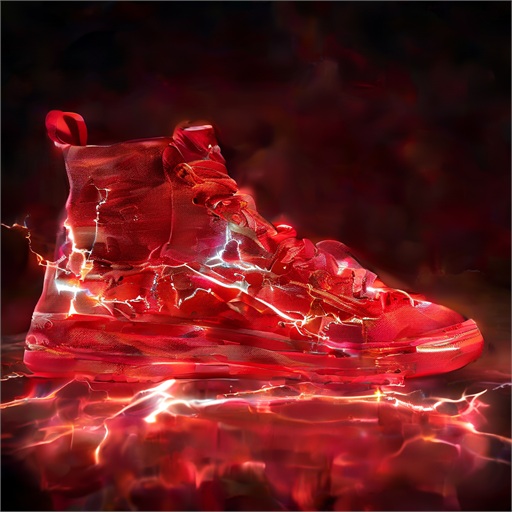} &
\includegraphics[width=\qualimgsize]{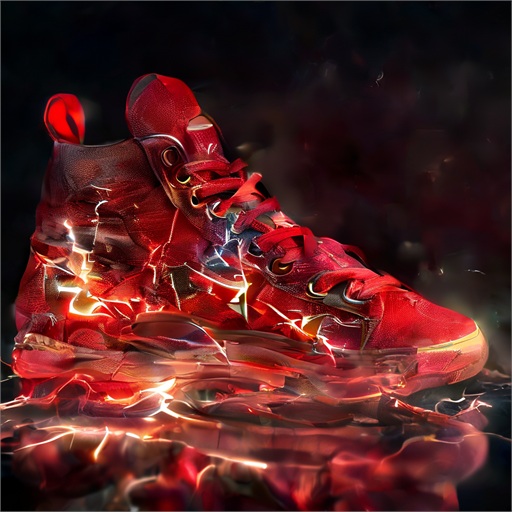} &
\includegraphics[width=\qualimgsize]{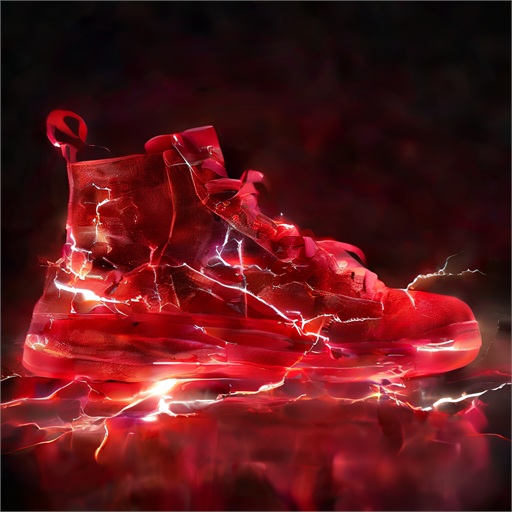} &
\includegraphics[width=\qualimgsize]{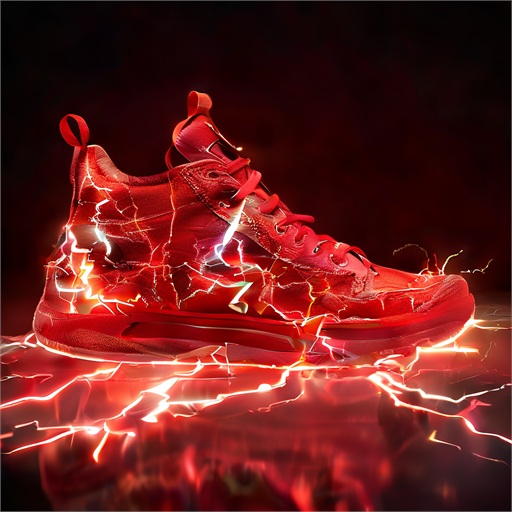} \\[-0.2em]

\promptline{red sneakers with lightning bolts on side Photorealistic, 8k }\\[0.4em]

\includegraphics[width=\qualimgsize]{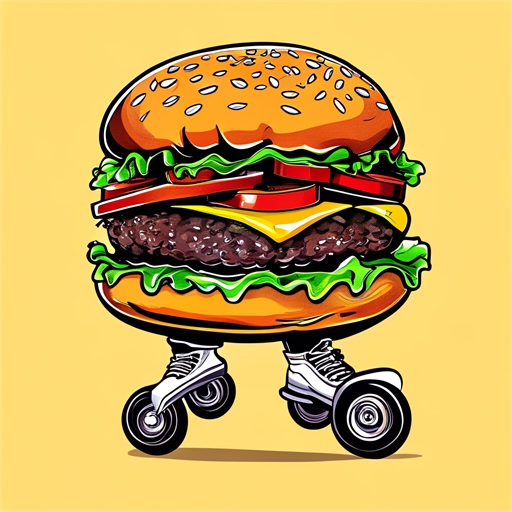} &
\imgvsep &
\includegraphics[width=\qualimgsize]{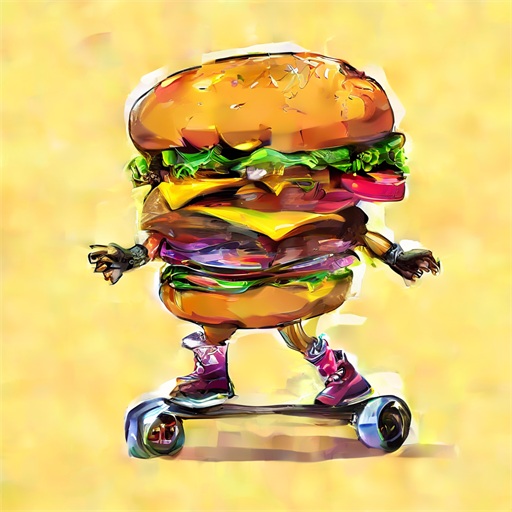} &
\includegraphics[width=\qualimgsize]{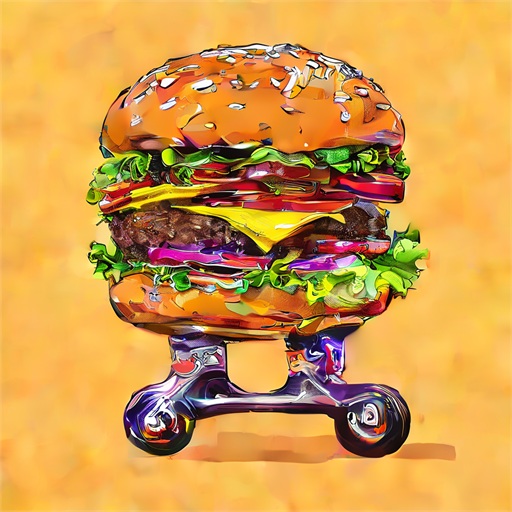} &
\includegraphics[width=\qualimgsize]{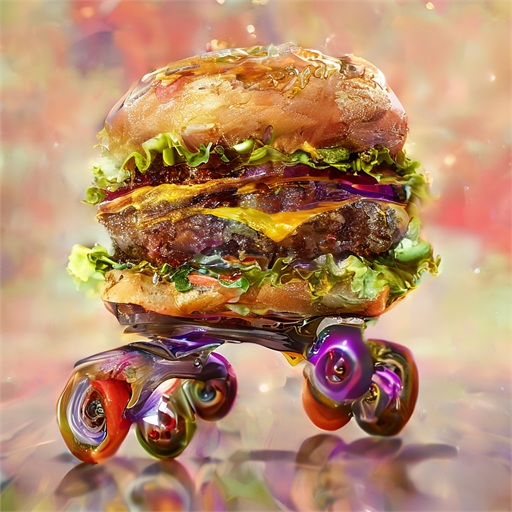} &
\includegraphics[width=\qualimgsize]{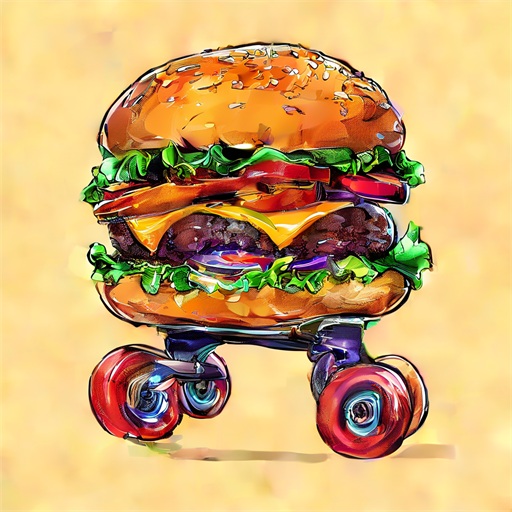} &
\includegraphics[width=\qualimgsize]{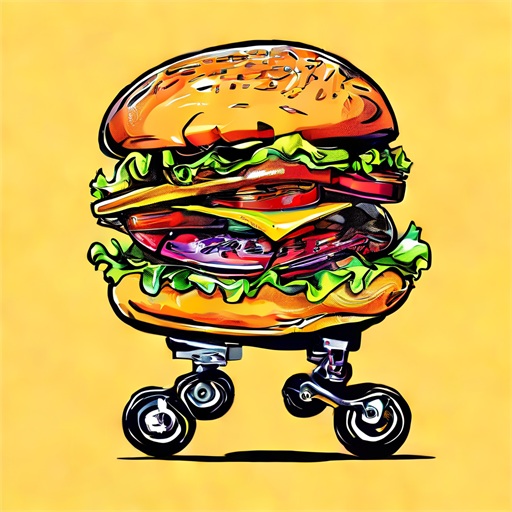} \\[-0.2em]

\promptline{hamburger roller skating}\\[0.4em]

\includegraphics[width=\qualimgsize]{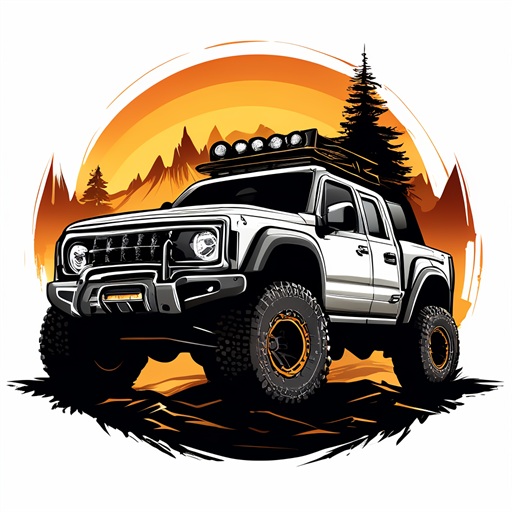} &
\imgvsep &
\includegraphics[width=\qualimgsize]{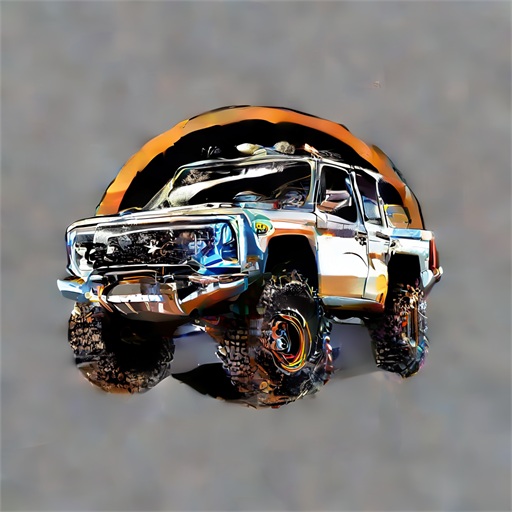} &
\includegraphics[width=\qualimgsize]{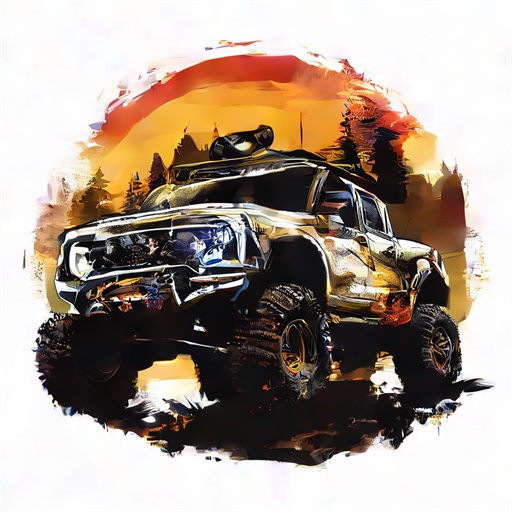} &
\includegraphics[width=\qualimgsize]{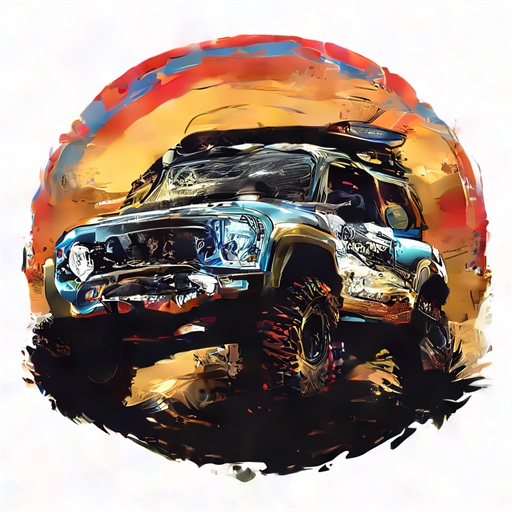} &
\includegraphics[width=\qualimgsize]{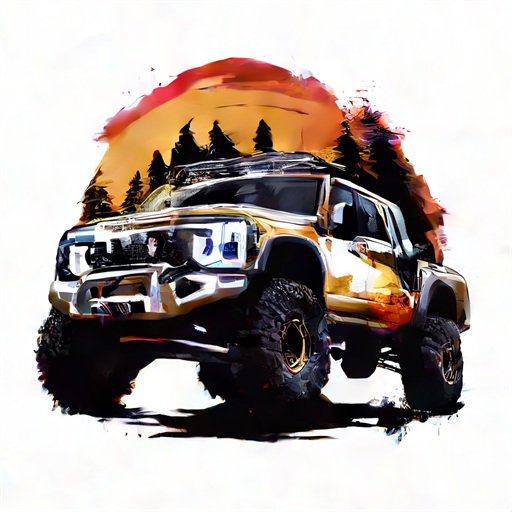} &
\includegraphics[width=\qualimgsize]{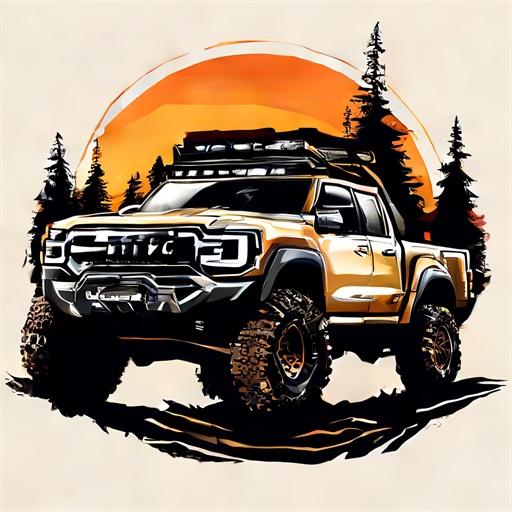} \\[-0.2em]

\promptline{pickup offroad logo}\\[-0.2em]

\end{tabular}
\end{adjustbox}

\vspace{-1mm}
\captionof{figure}{
Qualitative comparison on SANA-1.5-1.6B. The dashed line separates BF16 from the quantized methods. All quantized methods are evaluated under the W3A3 setting.
}
\label{fig:sana_additional_vis}

\vspace{1.5mm}

\begin{adjustbox}{width=\textwidth}
\begin{tabular}{c@{\hspace{0.5mm}}c c c c c c}

\textbf{BF16} & &
\textbf{OmniQuant} &
\textbf{SVDQuant} &
\textbf{ConvRot} &
\textbf{PermuQuant} &
\textbf{PermuQuant-H} \\

\includegraphics[width=\qualimgsize]{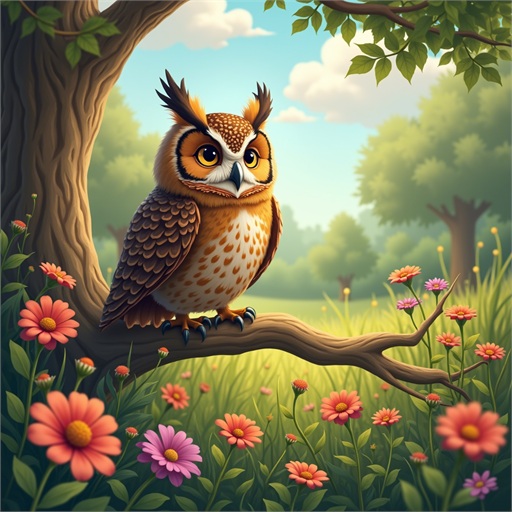} &
\imgvsep &
\includegraphics[width=\qualimgsize]{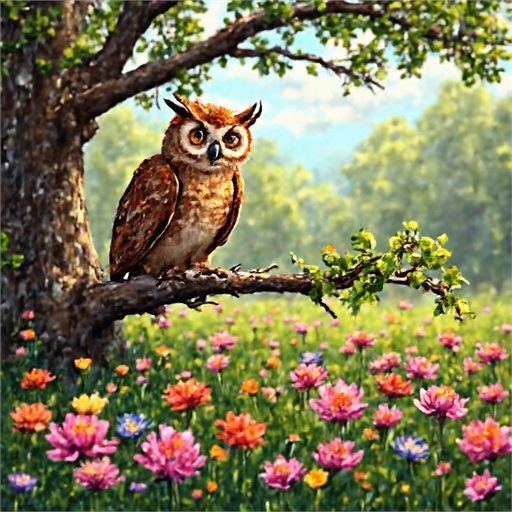} &
\includegraphics[width=\qualimgsize]{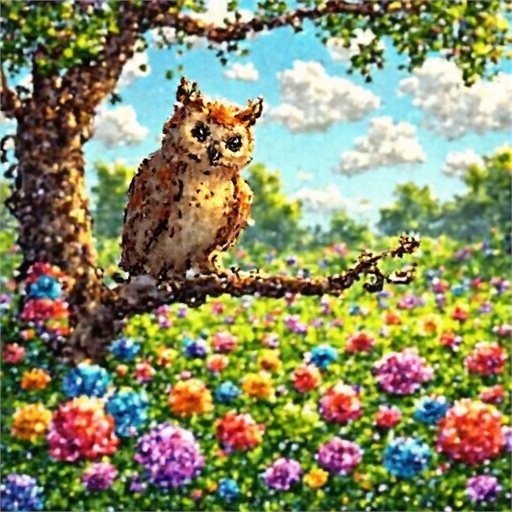} &
\includegraphics[width=\qualimgsize]{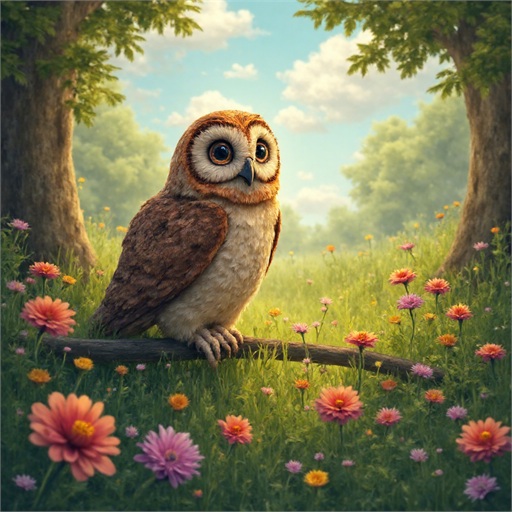} &
\includegraphics[width=\qualimgsize]{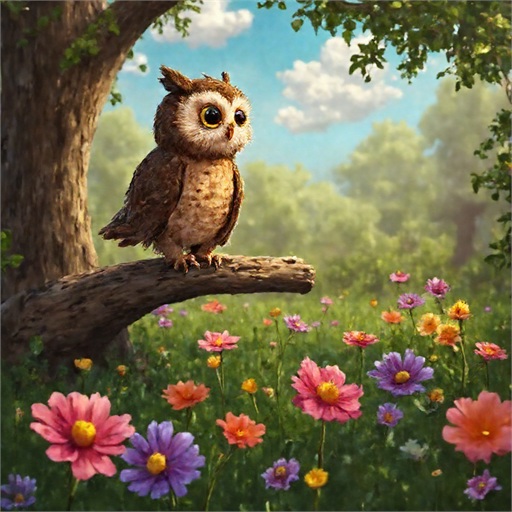} &
\includegraphics[width=\qualimgsize]{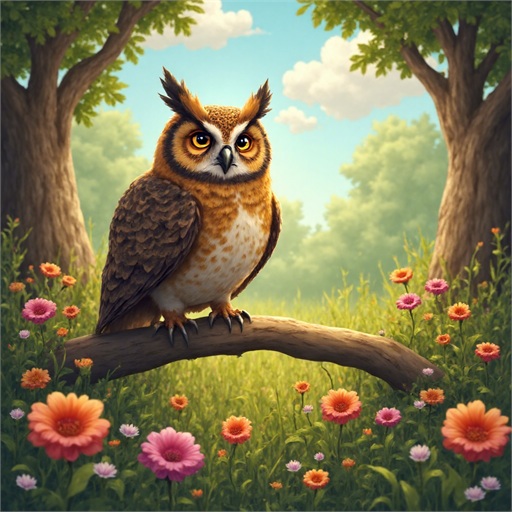} \\[-0.2em]

\promptline{realistic, an owl on a tree branch in a meadow surrounded by colorful flowers}\\[0.4em]

\includegraphics[width=\qualimgsize]{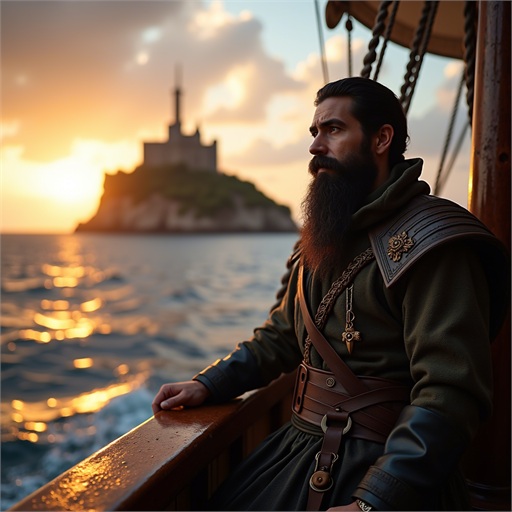} &
\imgvsep &
\includegraphics[width=\qualimgsize]{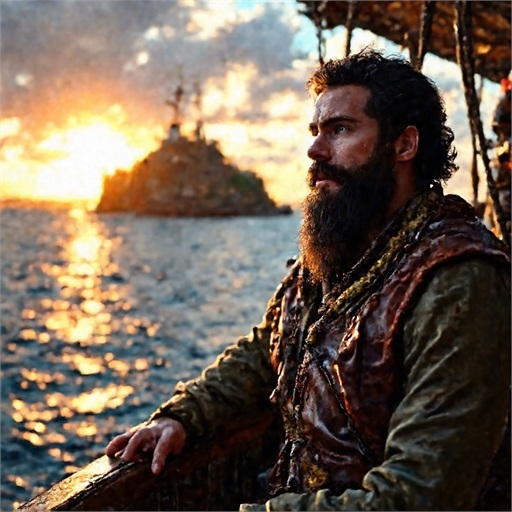} &
\includegraphics[width=\qualimgsize]{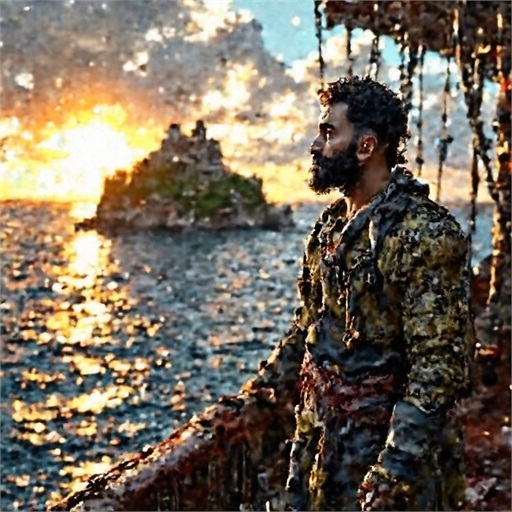} &
\includegraphics[width=\qualimgsize]{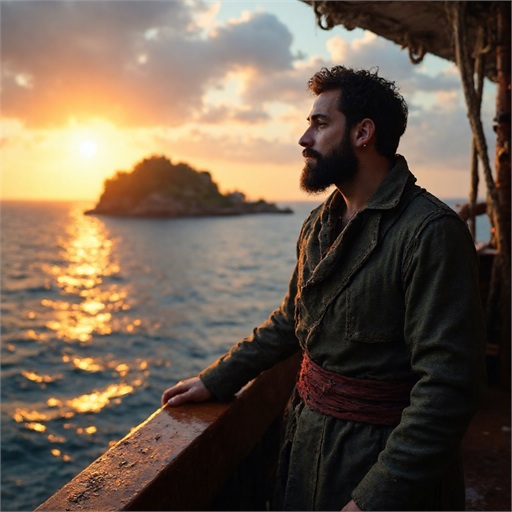} &
\includegraphics[width=\qualimgsize]{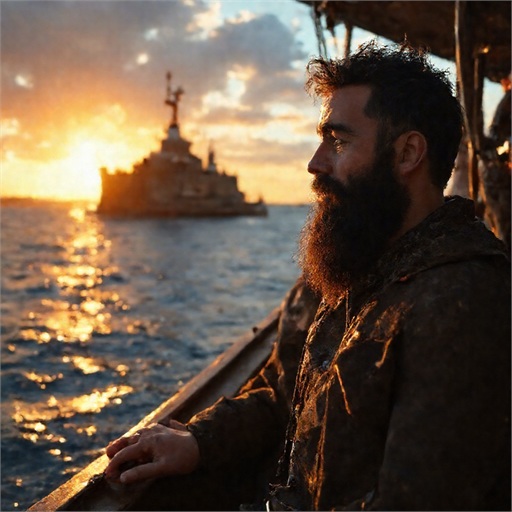} &
\includegraphics[width=\qualimgsize]{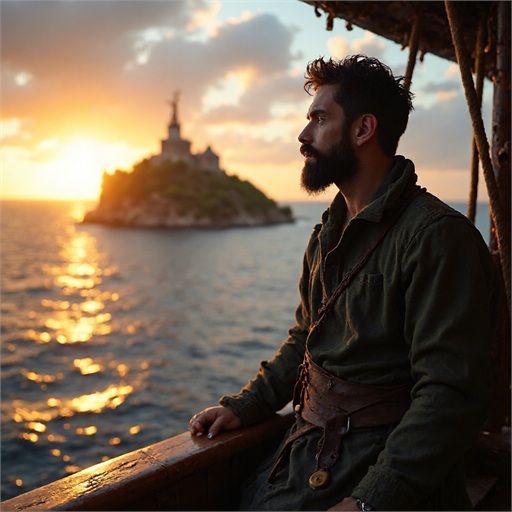} \\[-0.2em]

\promptline{strong man black hair beard on ship looks at an island in the distance sunbeams romantic}\\[0.4em]

\includegraphics[width=\qualimgsize]{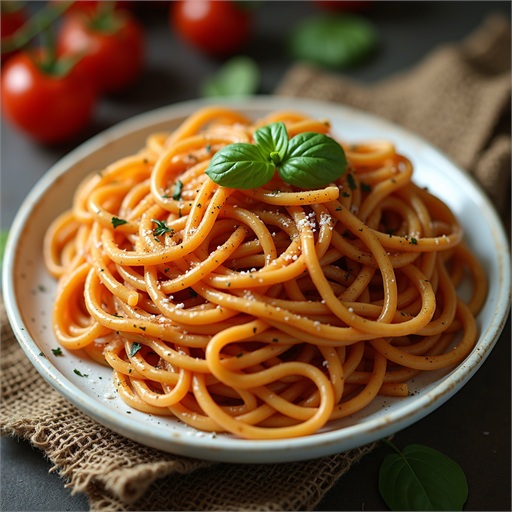} &
\imgvsep &
\includegraphics[width=\qualimgsize]{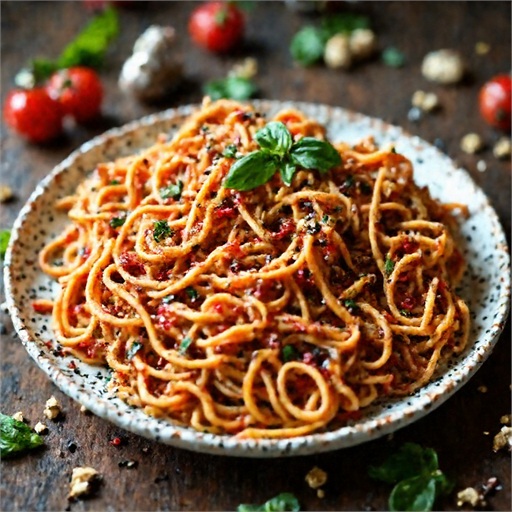} &
\includegraphics[width=\qualimgsize]{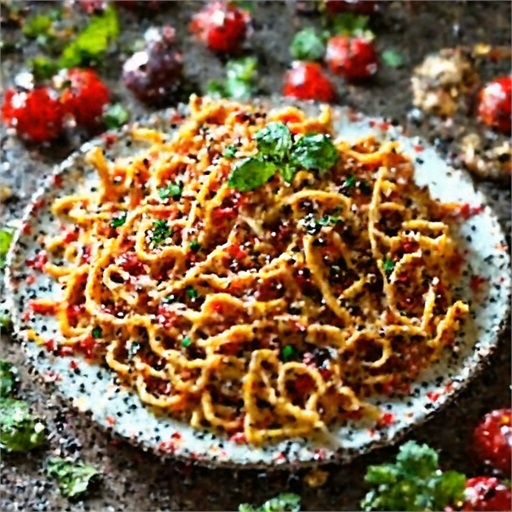} &
\includegraphics[width=\qualimgsize]{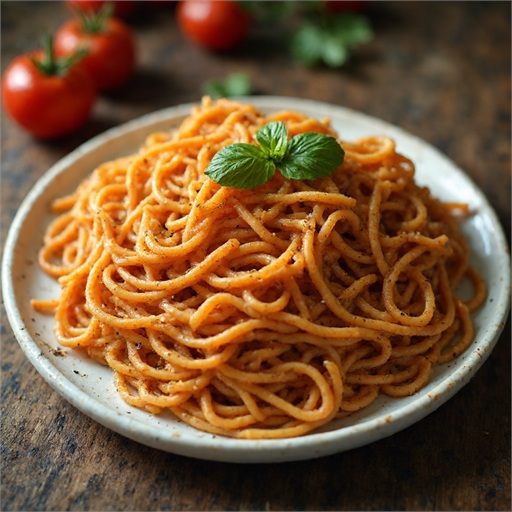} &
\includegraphics[width=\qualimgsize]{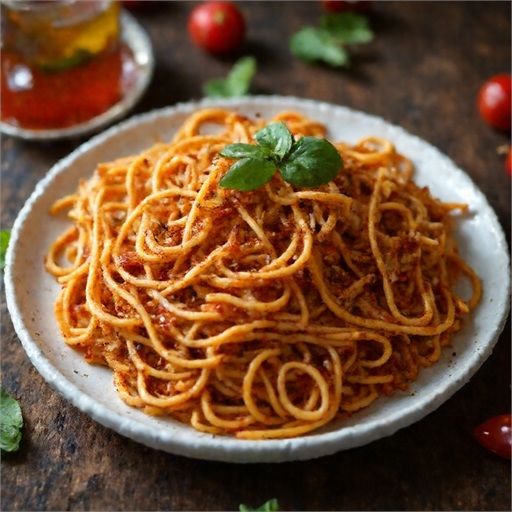} &
\includegraphics[width=\qualimgsize]{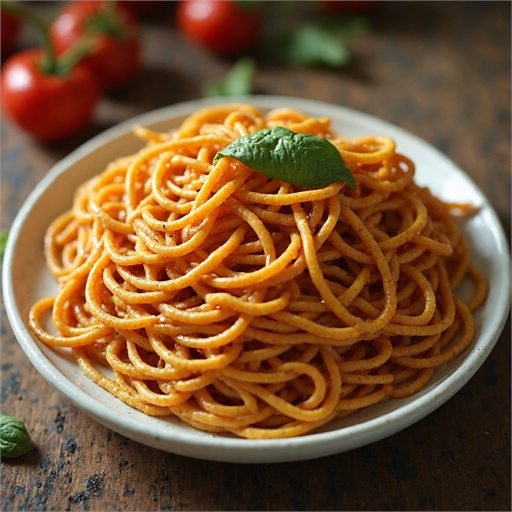} \\[-0.2em]

\promptline{spaghetti, photorealistic, very detailed }\\[0.4em]

\includegraphics[width=\qualimgsize]{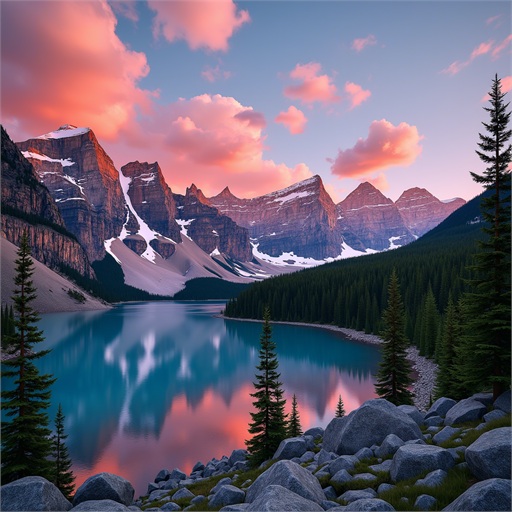} &
\imgvsep &
\includegraphics[width=\qualimgsize]{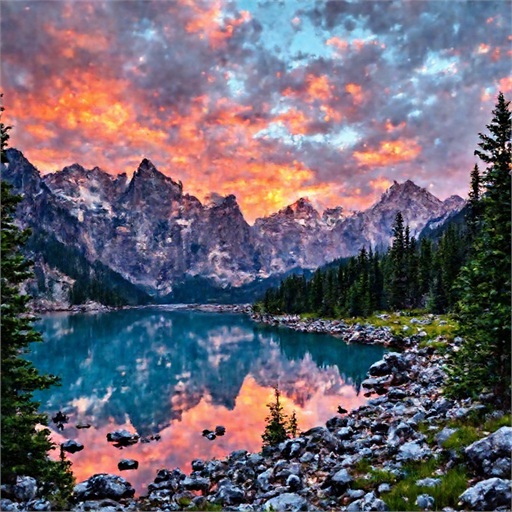} &
\includegraphics[width=\qualimgsize]{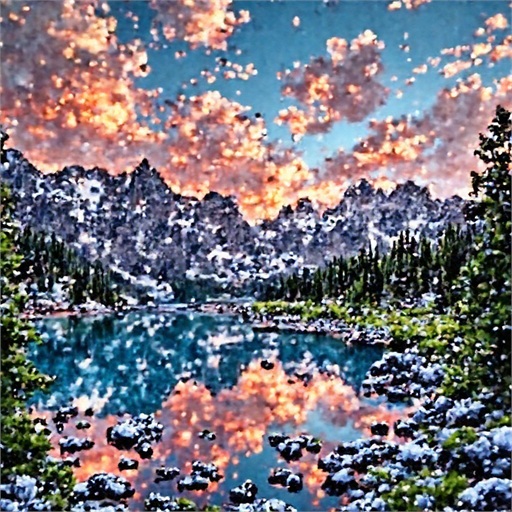} &
\includegraphics[width=\qualimgsize]{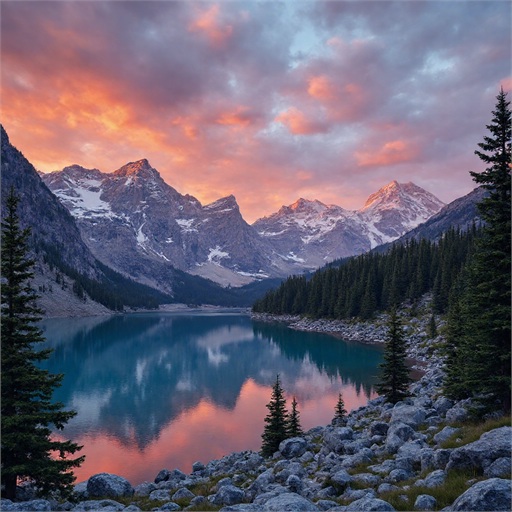} &
\includegraphics[width=\qualimgsize]{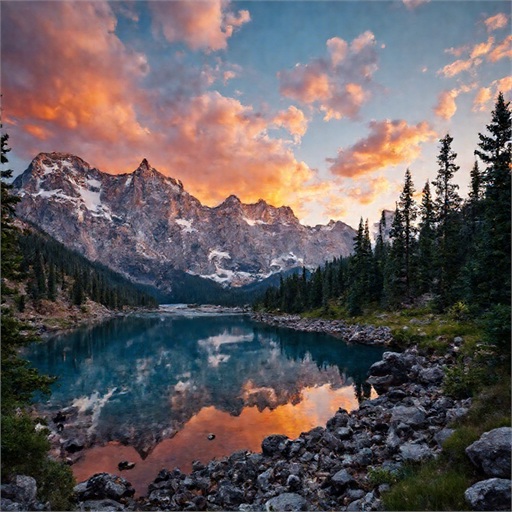} &
\includegraphics[width=\qualimgsize]{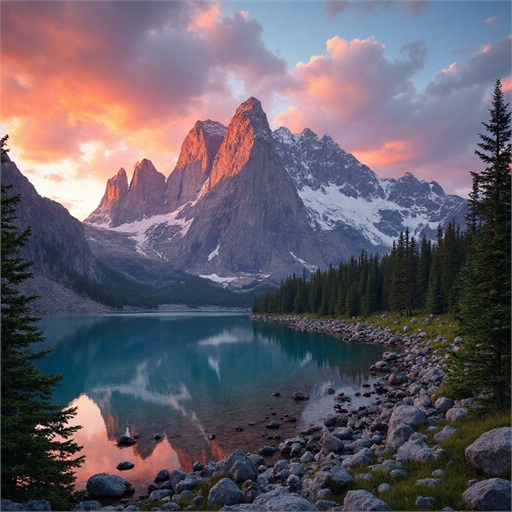} \\[-0.2em]

\promptline{canadian rockies Jasper, colorful, colorful sky, 8k, photorealistic, hyper  detailed, intricate details, Shot on DIGITAL CINEMA V  RAPTOR XL 8K VV Cinema Camera, Shutter Speed 1 800, raw, super resolution, tone mapping, ray tracing, Megapixels}\\[-0.2em]

\end{tabular}
\end{adjustbox}

\vspace{-1mm}
\captionof{figure}{
Qualitative comparison on FLUX.1-dev. The dashed line separates BF16 from the quantized methods. All quantized methods are evaluated under the W3A3 setting.
}
\label{fig:flux_additional_vis}

\vspace{-5mm}
\end{figure*}

\begin{figure*}[t]
\centering
\scriptsize
\setlength{\tabcolsep}{2pt}
\renewcommand{\arraystretch}{0.95}

\setlength{\qualimgsize}{0.16\textwidth}

\newcommand{\imgvsep}{%
  \begin{tikzpicture}[baseline={(0,0)}]
    \draw[densely dashed, line width=0.4pt] (0,0) -- (0,\qualimgsize);
  \end{tikzpicture}%
}

\newcommand{\promptline}[1]{%
\multicolumn{7}{c}{%
\parbox{0.92\textwidth}{%
\centering
\setlength{\baselineskip}{0.85\baselineskip}%
\textit{Prompt: #1}%
}%
}%
}

\begin{adjustbox}{width=\textwidth}
\begin{tabular}{c@{\hspace{0.5mm}}c c c c c c}

\textbf{BF16} & &
\textbf{OmniQuant} &
\textbf{SVDQuant} &
\textbf{ConvRot} &
\textbf{PermuQuant} &
\textbf{PermuQuant-H} \\

\includegraphics[width=\qualimgsize]{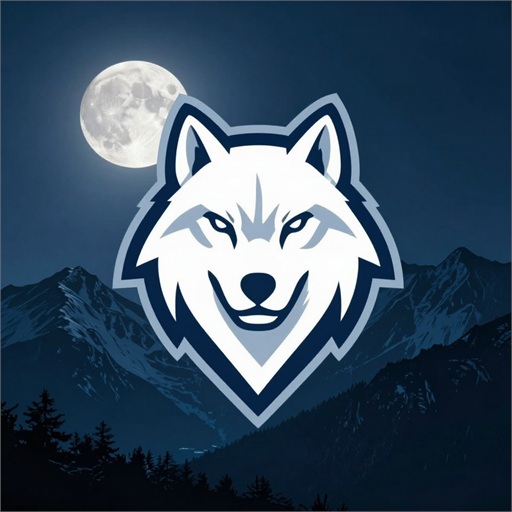} &
\imgvsep &
\includegraphics[width=\qualimgsize]{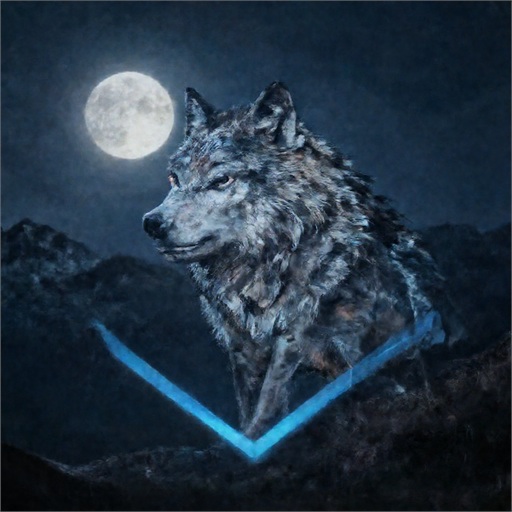} &
\includegraphics[width=\qualimgsize]{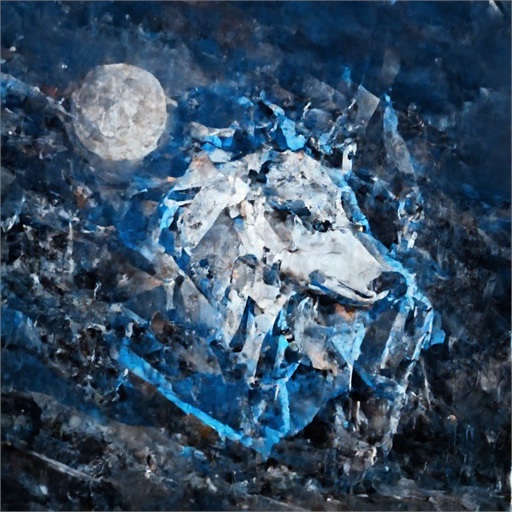} &
\includegraphics[width=\qualimgsize]{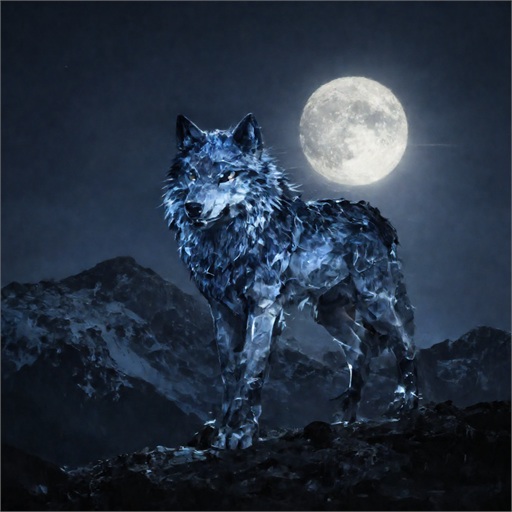} &
\includegraphics[width=\qualimgsize]{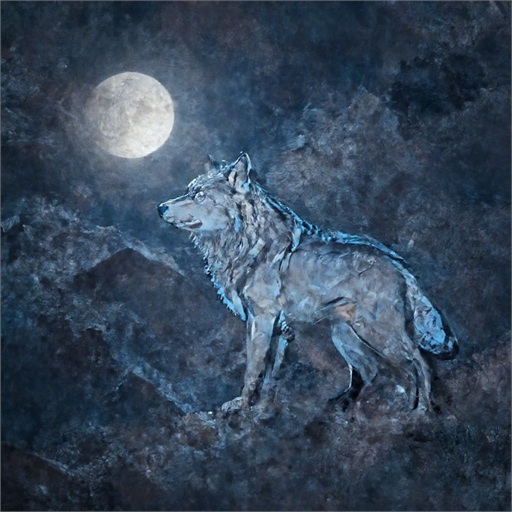} &
\includegraphics[width=\qualimgsize]{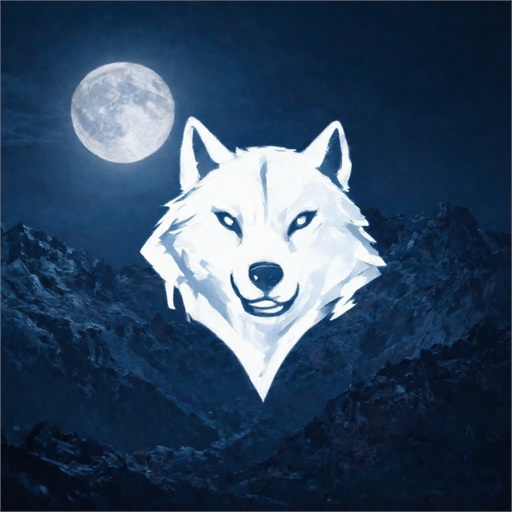} \\[-0.2em]

\promptline{a logo of wolf, blue light shadow, ultra realistic, 4k hd, full moon, mountains}\\[0.4em]

\includegraphics[width=\qualimgsize]{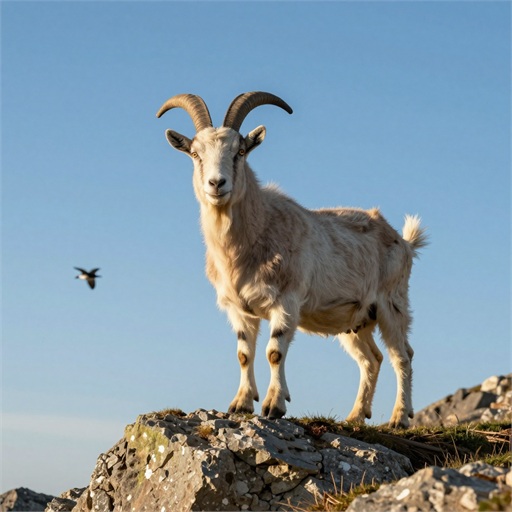} &
\imgvsep &
\includegraphics[width=\qualimgsize]{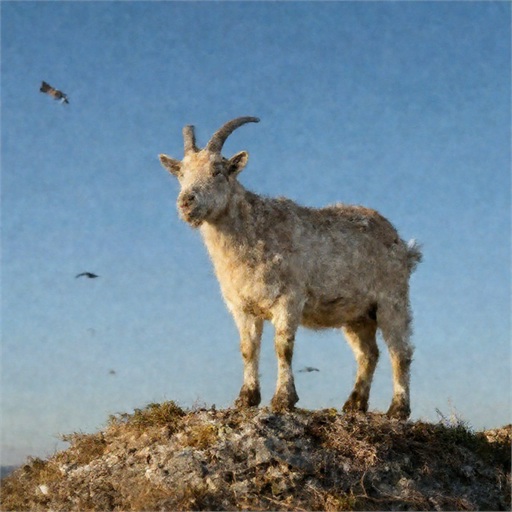} &
\includegraphics[width=\qualimgsize]{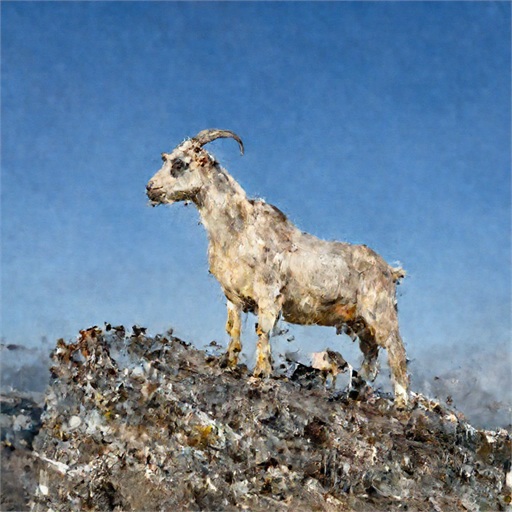} &
\includegraphics[width=\qualimgsize]{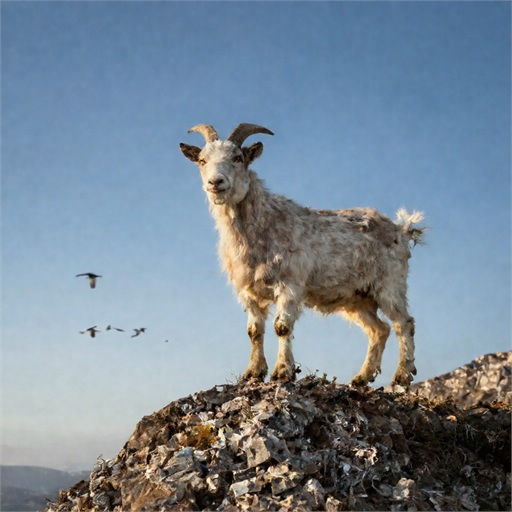} &
\includegraphics[width=\qualimgsize]{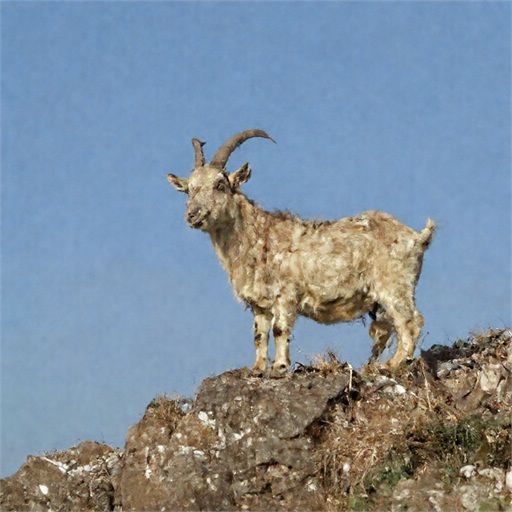} &
\includegraphics[width=\qualimgsize]{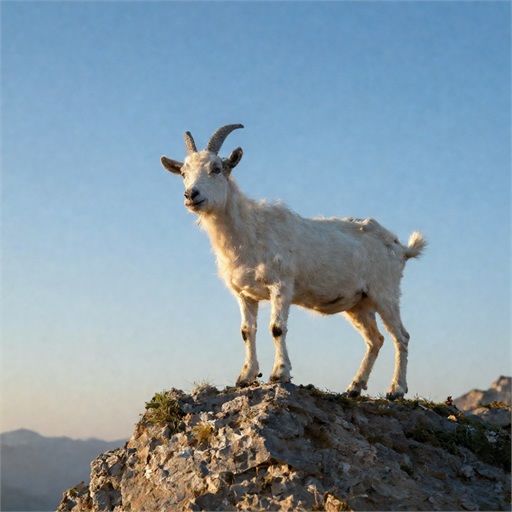} \\[-0.2em]

\promptline{a big strong goat on the top of the mountain, afternoon time, birds in the background, sun and the blue sky}\\[0.4em]

\includegraphics[width=\qualimgsize]{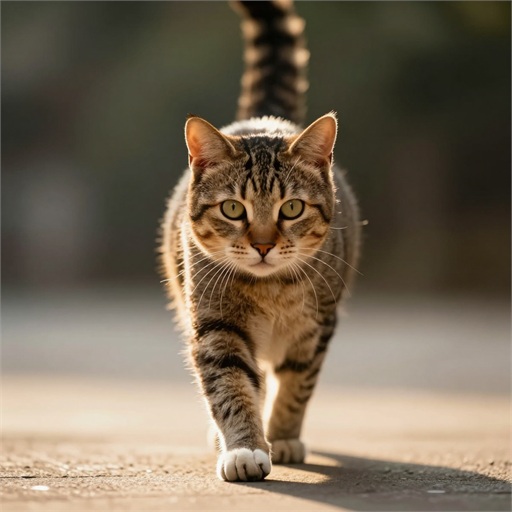} &
\imgvsep &
\includegraphics[width=\qualimgsize]{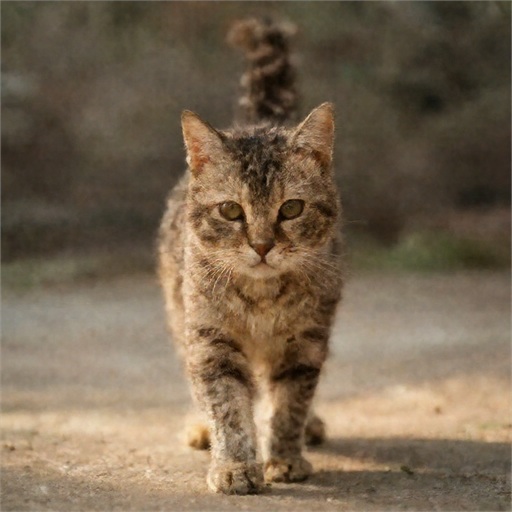} &
\includegraphics[width=\qualimgsize]{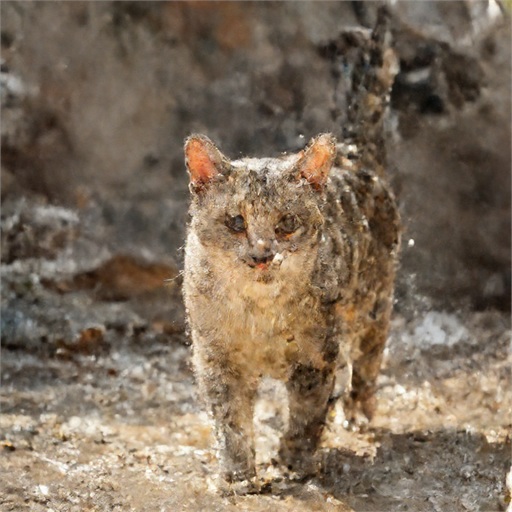} &
\includegraphics[width=\qualimgsize]{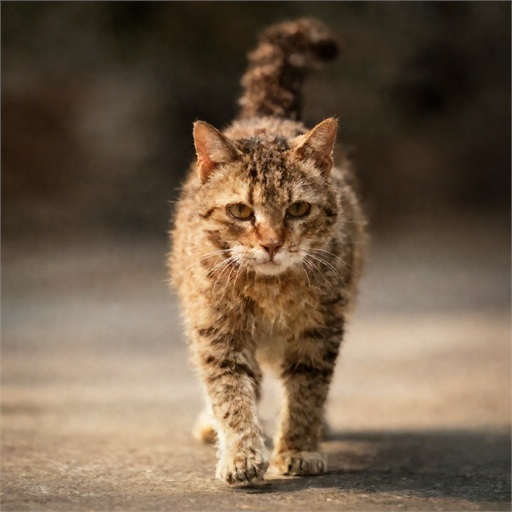} &
\includegraphics[width=\qualimgsize]{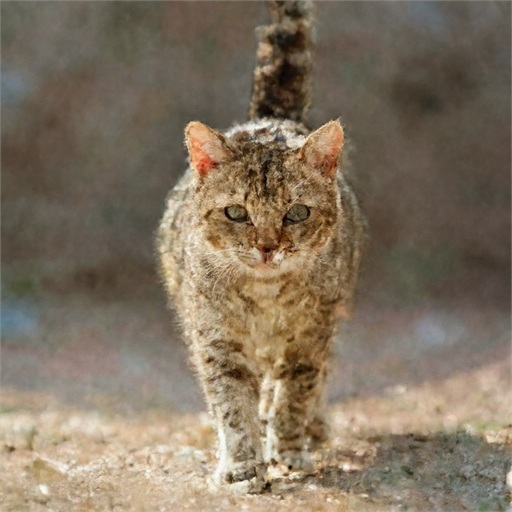} &
\includegraphics[width=\qualimgsize]{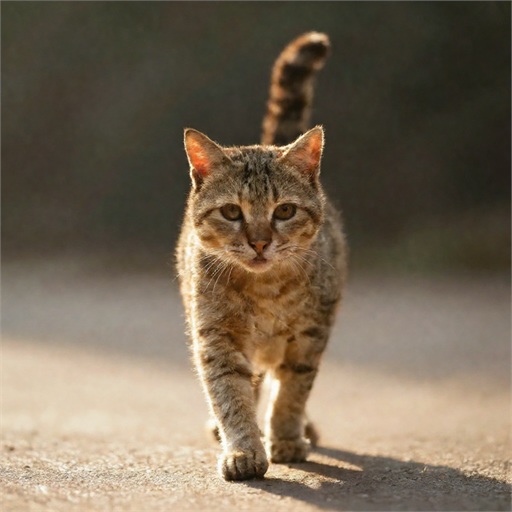} \\[-0.2em]

\promptline{the most amazing image of a cat walking toward you, backlit warm colors}\\[0.4em]

\includegraphics[width=\qualimgsize]{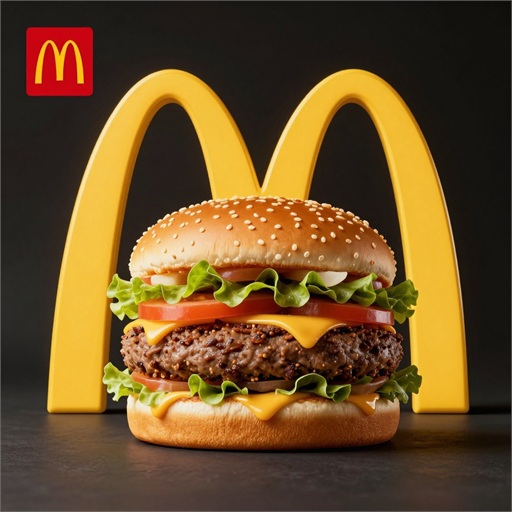} &
\imgvsep &
\includegraphics[width=\qualimgsize]{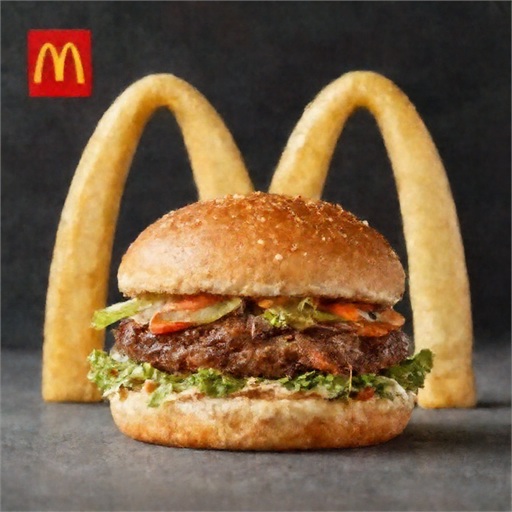} &
\includegraphics[width=\qualimgsize]{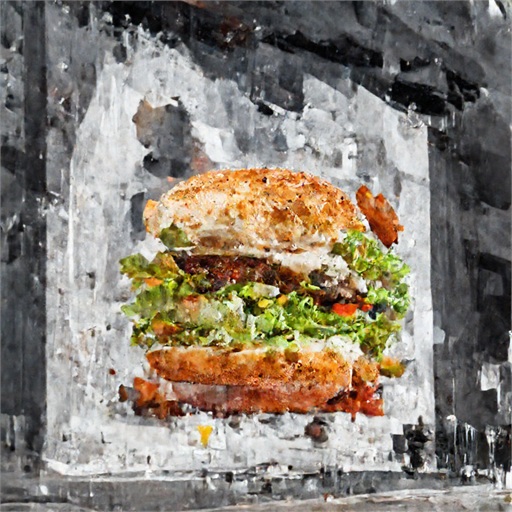} &
\includegraphics[width=\qualimgsize]{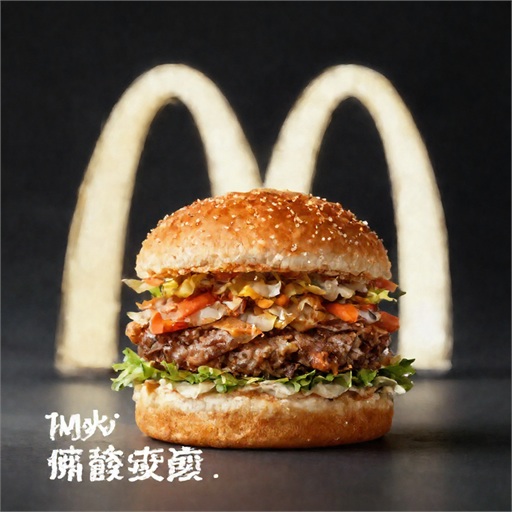} &
\includegraphics[width=\qualimgsize]{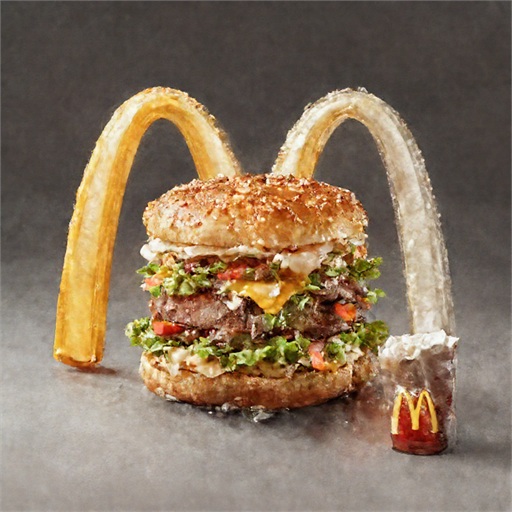} &
\includegraphics[width=\qualimgsize]{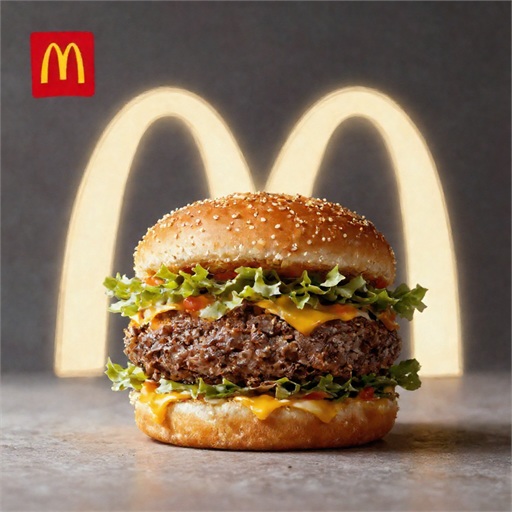} \\[-0.2em]

\promptline{a modern digital advertisement for a brand new never before seen hamburger made from an unknown material at McDonalds, clean modern design, advertisement, branding and marketing, modern photography}\\[-0.2em]

\end{tabular}
\end{adjustbox}

\vspace{-1mm}
\caption{
Qualitative comparison on Z-Image-Turbo. The dashed line separates BF16 from the quantized methods. All quantized methods are evaluated under the W3A3 setting.
}
\label{fig:zimage_additional_vis}
\vspace{-5mm}
\end{figure*}

\end{document}